\documentclass{article}

\usepackage[preprint]{neurips_2020}
\usepackage[utf8]{inputenc} 
\usepackage[T1]{fontenc}    
\usepackage{hyperref}       
\usepackage{url}            
\usepackage{booktabs}       
\usepackage{amsfonts}       
\usepackage{nicefrac}       
\usepackage{microtype}      

\usepackage{amssymb,amsmath,amsthm,enumerate,threeparttable,bm,subfigure,graphicx}
\usepackage{algorithm, algorithmic}
\usepackage{booktabs}
\usepackage{multirow}
\newtheorem{lemma}{Lemma}
\newtheorem{theorem}{Theorem}
\newtheorem{definition}{Definition}

\newtheorem{property}{Property}
\usepackage{color}
\usepackage{appendix}

\renewcommand{\exp}[1]{\operatorname{exp}\left(#1\right)}

\newcommand{\E}[1]{\mathbb{E}\left[#1\right]}

\newcommand{\1}[1]{\mathbb{I}\left\{#1\right\}}

\long\def\comment#1{}

\title{Multinomial Random Forest: Toward Consistency and Privacy-Preservation}

\author{%
	Yiming Li$^{1, }$\thanks{equal contribution.} , Jiawang Bai$^{1, 2, *}$, Jiawei Li$^{1}$, Yang Xue$^1$, Yong Jiang$^{1, 2}$, Chun Li$^3$, Shutao Xia$^{1,2}$\\
	$^1$Tsinghua Shenzhen International Graduate School,  Tsinghua University, China\\
	$^2$PCL Research Center of Networks and Communications, Peng Cheng Laboratory, China\\
	$^3$Institute for Computational Biology, Case Western Reserve University, Ohio, USA\\
	\texttt{xueyang.swjtu@gmail.com}; \texttt{xiast@sz.tsinghua.edu.cn}\\
}

\begin{document}

\maketitle

\begin{abstract}
Despite the impressive performance of random forests (RF), its theoretical properties have not been thoroughly understood. In this paper, we propose a novel RF framework, dubbed multinomial random forest (MRF), to analyze the \emph{consistency} and \emph{privacy-preservation}. Instead of deterministic greedy split rule or with simple randomness, the MRF adopts two impurity-based multinomial distributions to randomly select a split feature and a split value respectively. Theoretically, we prove the consistency of the proposed MRF and analyze its privacy-preservation within the framework of differential privacy. We also demonstrate with multiple datasets that its performance is on par with the standard RF. To the best of our knowledge, MRF is the first consistent RF variant that has comparable performance to the standard RF.
\end{abstract}

\section{Introduction}

Random forest (RF) \cite{Breiman2001} is a popular type of ensemble learning method. Because of its excellent performance and fast yet efficient training process, the standard RF and its several variants have been widely used in many fields, such as computer vision \cite{cootes2012,kontschieder2016} and data mining \cite{bifet2009,xiong2012}. However, due to the inherent bootstrap randomization and the highly greedy data-dependent construction process, it is very difficult to analyze the theoretical properties of random forests \cite{biau2012}, especially for the \emph{consistency}. Since consistency ensures that the model goes to optimal under a sufficient amount of data, this property is especially critical in this big data era.

To address this issue, several RF variants \cite{breiman2004,biau2008,genuer2012,biau2012,denil2014,Wang2018} were proposed. Unfortunately, all existing consistent RF variants suffer from relatively poor performance compared with the standard RF due to two mechanisms introduced for ensuring consistency. On the one hand, the data partition process allows only half of the training samples to be used for the construction of tree structure, which significantly reduces the performance of consistent RF variants. On the other hand, extra randomness ($e.g.$, Poisson or Bernoulli distribution) is introduced, which further hinders the performance. Accordingly, those mechanisms introduced for theoretical analysis makes it difficult to eliminate the performance gap between consistent RF and standard RF.

\begin{figure}[ht]
 \centering
 \vspace{-0.5em}
 \includegraphics[width=0.48\textwidth]{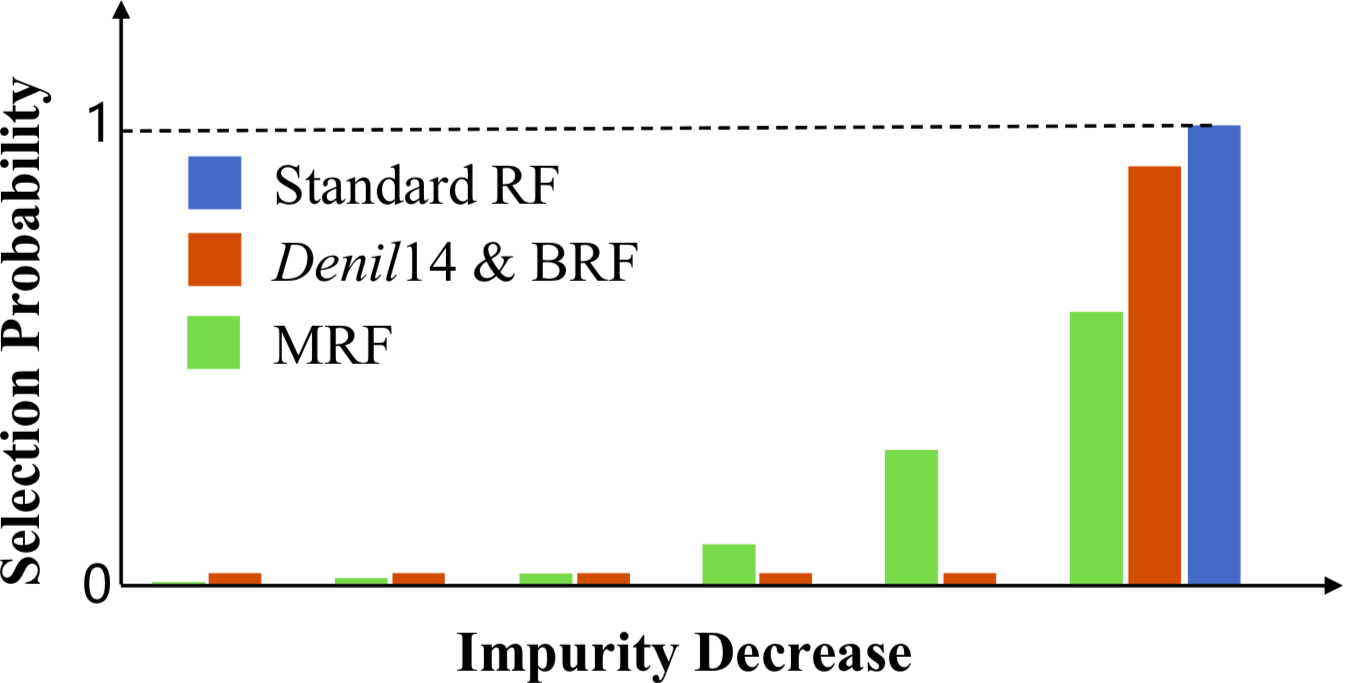}
 \caption{Splitting criteria of different RFs. For standard RF, it always chooses the split point with the highest impurity decrease. For {\it Denil14} and BRF, they choose the split point also in a greedy way mostly, while holding a small or even negligible probability in selecting another point randomly. The selection probability in MRF is positively related to the impurity decrease. All randomness in consistent RF variants is aiming to fulfill the consistency, whereas the one in MRF is more reasonable.}
 \label{prob}
\vspace{-0.5em}
\end{figure}

Is this gap really impossible to fill? In this paper, we propose a novel consistent RF framework, dubbed multinomial random forest (MRF), by introducing the randomness more reasonably, as shown in Figure \ref{prob}. In the MRF, two impurity-based multinomial distributions are used as the basis for randomly selecting a split feature and a specific split value respectively. Accordingly, the ``best'' split point has the highest probability to be chosen, while other candidate split points that are nearly as good as the ``best'' one will also have a good chance to be selected. This randomized splitting process is more reasonable and makes up the accuracy drop with almost no extra computational complexity. Besides, the introduced impurity-based randomness is essentially an exponential mechanism satisfying differential privacy, and the randomized prediction of each tree proposed in this paper also adopts the exponential mechanism. Accordingly, we can also analyze the privacy-preservation of MRF under the differential privacy framework. To the best of our knowledge, this privacy-preservation property, which is important since the training data may well contains sensitive information, has never been analyzed by previous consistent RF variants.  

The main contributions of this work are three-fold: {\bf 1)} we propose a novel multinomial-based method to improve the greedy split process of decision trees; {\bf 2)} we propose a new random forests variant, dubbed multinomial random forest (MRF), based on which we analyze its consistency and privacy-preservation; {\bf 3)} extensive experiments demonstrate that the performance of MRF is on par with Breiman's original RF and is better than all existing consistent RF variants. MRF is the first consistent RF variant that simultaneously has performance comparable to the standard RF.

\section{Related Work}
\subsection{Consistent Random Forests}
Random forest \cite{Breiman2001} is a distinguished ensemble learning algorithm inspired by the random subspace method \cite{ho1998} and random split selection \cite{dietterich2000}. The standard decision trees are built upon bootstrap datasets and splitting with the CART methodology \cite{breiman1984}. Its various variants, such as quantile regression forests \cite{meinshausen2006} and deep forests \cite{zhou2017}, were proposed and used in a wide range of applications \cite{bifet2009, cootes2012,kontschieder2016} for their effective training process and great interpretability. Despite the widespread use of random forests in practice, theoretical analysis of their success has yet been fully established. Breiman \cite{Breiman2001} showed the first theoretical result indicating that the generalization error is bounded by the performance of individual tree and the diversity of the whole forest. After that, the relationship between random forests and a type of nearest neighbor-based estimator was also studied \cite{lin2006}.

One of the important properties, the consistency, has yet to be established for random forests. Consistency ensures that the result of RF converges to the optimum as the sample size increases, which was first discussed by Breiman \cite{breiman2004}. As an important milestone, Biau \cite{biau2008} proved the consistency of two directly simplified random forest. Subsequently, several consistent RF variants were proposed for various purposes; for example, random survival forests \cite{Ishwaran2010}, an online version of random forests variant \cite{denil2013} and a generalized regression forests \cite{athey2019}. Recently, Haghiri  \cite{haghiri2018} proposed CompRF, whose split process relied on triplet comparisons rather than information gain. To ensure the consistency, \cite{biau2012} suggested that an independent dataset is needed to fit in the leaf. This approach is called data partition. Under this framework, Denil \cite{denil2014} developed a consistent RF variant (called {\it Denil14} in this paper) narrowing the gap between theory and practice. Following {\it Denil14}, Wang \cite{Wang2018} introduced Bernoulli random forests (BRF), which reached state-of-the-art performance. The comparison of different consistent RFs is included in the \textbf{Appendix}.

Although several consistent RF variants are proposed, due to the relatively poor performance compared with standard RF, how to fulfill the gap between theoretical consistency and the performance in practice is still an important open problem.

\subsection{Privacy-Preservation}
In addition to the exploration of consistency, some schemes \cite{MohammedCFY11, PatilS14} were also presented to address privacy concerns. Among those schemes, differential privacy \cite{Dwork06}, as a new and promising privacy-preservation model, has been widely adopted in recent years. In what follows, we outline the basic content of differential privacy. 

Let $\mathcal{D}=\{(\bm{X}_{i},Y_{i})\}_{i=1}^{n}$ denotes a dataset consisting of $n$ $i.i.d.$ observations, where $\bm{X}_{i} \in \mathbb{R}^D$ represents $D$-dimensional features and $Y_{i} \in \{1, \ldots, K\}$ indicates the label.  Let $\mathcal{A}=\{A_{1}, A_{2},\ldots,  A_{D}\}$ represents the feature set. The formal definition of differential privacy is detailed as follow.

\begin{definition}[\textbf{$\epsilon$-Differential Privacy}]
A randomized mechanism $\mathcal{M}$ gives $\epsilon$-differential privacy for every set of outputs $O$ and any neighboring datasets $\mathcal{D}$ and $\mathcal{D}'$ differing in one record, if $\mathcal{M}$ satisfies: 
\begin{equation}\label{DP}
\Pr[\mathcal{M}(\mathcal{D})\in O]\leqslant \exp\epsilon\cdot \Pr[\mathcal{M}(\mathcal{D'})\in O],
\end{equation}
\end{definition}
where $\epsilon$ denotes the privacy budget that restricts the privacy guarantee level of $\mathcal{M}$. A smaller $\epsilon$ represents a stronger privacy level. 

Currently, two basic mechanisms are widely used to guarantee differential privacy: the Laplace mechanism \cite{DworkMNS06} and the Exponential mechanism \cite{McSherryT07}, where the former one is suitable for numeric queries and the later is suitable for non-numeric queries. Since the MRF mainly involves selection operations, we adopt the exponential mechanism to preserve privacy. 

\begin{definition}[Exponential Mechanism]
Let $q: (\mathcal{D}, o)\rightarrow\mathbb{R}$ be a score function of dataset $\mathcal{D}$ that measures the quality of output $o\in O$. The exponential mechanism $\mathcal{M}(\mathcal{D})$ satisfies $\epsilon$-differential privacy, if it outputs $o$ with probability proportional to $\exp {\frac{\epsilon q(\mathcal{D}, o)}{2\triangle q}}$, $i.e.$, 
\begin{equation}\label{EM}
\Pr[\mathcal{M}(\mathcal{D})=o]=\frac{\exp {\frac{\epsilon q(\mathcal{D}, o)}{2\triangle q}}}{\sum_{o'\in O}\exp {\frac{\epsilon q(\mathcal{D}, o')}{2\triangle q}}},
\end{equation}
\end{definition}
where $\triangle q$ is the sensitivity of the quality function, defined as 
\begin{equation}\label{sensitivity}
\triangle q=\max_{\forall o, \mathcal{D}, \mathcal{D'}}\left|q(\mathcal{D}, o)-q(\mathcal{D'}, o)\right|.
\end{equation}

\section{Multinomial Random Forests}\label{class}

\subsection{Training Set Partition}\label{TDSP}
Compared to the standard RF, the MRF replaces the bootstrap technique by a partition of the training set, which is necessary for consistency, as suggested in \cite{biau2012}. Specifically, to build a tree, the training set $\mathcal{D}$ is divided randomly into two non-overlapping subsets $\mathcal{D}^S$ and $\mathcal{D}^E$, which play different roles. One subset $\mathcal{D}^S$ will be used to build the structure of a tree; we call the observations in this subset the \textbf{structure points}. Once a tree is built, the labels of its leaves will be re-determined on the basis of another subset $\mathcal{D}^E$; we call the observations in the second subset the \textbf{estimation points}. The ratio of two subsets is parameterized by \textbf{partition rate} = $|\mathrm{Structure\, points}|/|\mathrm{Estimation\, points}|$. To build another tree, the training set is re-partitioned randomly and independently.

\subsection{Tree Construction}\label{TC}
The construction of a tree relies on a recursive partitioning algorithm. Specifically, to split a node, we introduce two impurity-based multinomial distributions: one for split feature selection and another for split value selection. The specific split point is a pair of a split feature and a split value. In the classification problem, the impurity decrease at a node $u$ caused by a split point $v$ is defined as
\begin{equation}\label{1}
    I(\mathcal{D}_{u}^S, v) = T(\mathcal{D}_{u}^S) - \frac{|\mathcal{D}_{u}^{S_l}|}{|\mathcal{D}_{u}^S|}T(\mathcal{D}_{u}^{S_l})
    -\frac{|\mathcal{D}_{u}^{S_r}|}{|\mathcal{D}_{u}^S|}T(\mathcal{D}_{u}^{S_r}),
\end{equation}
where $\mathcal{D}_{u}^S$ is the subset of $\mathcal{D}^S$ at a node $u$, $\mathcal{D}_{u}^{S_l}$ and $\mathcal{D}_{u}^{S_r}$ generated by splitting $\mathcal{D}_{u}^S$ with $v$, are two subsets in the left child and right child of the node $u$, respectively, and $T(\cdot)$ is the impurity criterion ($e.g.$, Shannon entropy or Gini index).  Unless other specification, we ignore the subscript $u$ of each symbol, and use $I$ to denote $I(\mathcal{D}_{u}^S, v)$ for shorthand in the rest of this paper. 

Let $V=\{v_{ij}\}$ denote the set of all possible split points for the node and $I_{i,j}$ is the corresponding impurity decrease, where $v_{ij}$ is $i$-th value on the $j$-th feature.  In what follows, we first introduce the feature selection mechanism for a node, and then describe the split value selection mechanism corresponding to the selected feature. 

\textbf{$M(\bm{\phi})$-based split feature selection.} At first, we obtain a vector $\bm{I} =\left(I_1,\cdots, I_D\right) =\left(\max \limits_{i}\{I_{i,1}\}, \cdots, \max \limits_{i}\{I_{i,D}\}\right)$ based on each $I_{i,j}$, where  $\max \limits_{i}\{I_{i,j}\}$,  $j=1,\ldots, D$, is the largest possible impurity decrease of the feature $A_{j}$. Then, the following three steps need to be performed: 
\begin{itemize}
  \item Normalize $\bm{I}$: $\hat{\bm{I}} = \left(\frac{I_1-\min \bm{I}}{\max \bm{I} - \min \bm{I}}, \cdots, \frac{I_D-\min \bm{I}}{\max \bm{I} - \min \bm{I}}\right)$;
  \item Compute the probabilities $\bm{\phi}=(\phi_{1}, \cdots, \phi_{D})=\text{softmax}(\frac{B_1}{2}\hat{\bm{I}})$, where $B_1 \geq 0$ is a hyper-parameter related to privacy budget;
  \item Randomly select a feature according to the multinomial distribution $M(\bm{\phi})$.
\end{itemize}


\textbf{$M(\bm{\varphi})$-based split value selection.} After selecting the feature $A_{j}$ for a node, we need to determine the corresponding split value to construct two children. Suppose $A_{j}$ has $m$ possible split values, we need to perform the following steps:  
\begin{itemize}
 \item Normalize $\bm{I}^{(j)}=(I_{1,j}, \cdots, I_{m,j})$ as $\hat{\bm{I}}^{(j)} = \left(\frac{I_{1,j}-\min \bm{I}^{(j)}}{\max \bm{I}^{(j)} - \min \bm{I}^{(j)}}, \cdots, \frac{I_{m,j}-\min \bm{I}^{(j)}}{\max \bm{I}^{(j)} - \min \bm{I}^{(j)}}\right)$, where $j$ identifies the feature $A_{j}$;
\item Compute the probabilities $\bm{\varphi}=(\varphi_{1}, \cdots, \varphi_{m})=\text{softmax}(\frac{B_2}{2}\hat{\bm{I}}^{(j)})$, where $B_2 \geq 0$ is another hyper-parameter related to privacy budget;
\item Randomly select a split value according to the multinomial distribution $M(\bm{\varphi})$.
\end{itemize}

We repeat the above processes to split nodes until the stopping criterion is met. The stopping criterion relates to the minimum leaf size $k$. Specifically, the number of estimation points is required to be at least $k$ for every leaf. The pseudo code of training process is shown in the \textbf{Appendix}.

\subsection{Prediction}

Once a tree $h$ was grown based on $\mathcal{D}^S$, we re-determine the predicted values for leaves according to $\mathcal{D}^E$. Similar to \cite{Breiman2001},  given an unlabeled sample $\bm{x}$, we can easily know which leaf of $h$ it falls, and the empirical probability that sample $\bm{x}$ has label $c$ ($c\in \{1,\cdots,K\}$) is estimated to be
\begin{equation}\label{p1}
\eta^{(c)} (\bm{x}) = \frac{1}{|\mathcal{N}_h^E(\bm{x})|} \sum_{(\bm{X},Y) \in \mathcal{N}_h^E(\bm{x})} \1{Y = c},
\end{equation}
where $\mathcal{N}_h^E(\bm{x})$ is the set of estimation points in the leaf containing $\bm{x}$, and $\mathbb{I}(\cdot)$ is an indicator function. 

In contrast to the standard RF and consistent RF variants, the predicted label $h(\bm{x})$ of $\bm{x}$ is randomly selected with a probability proportional to $\exp {\frac{B_{3}\eta^{(c)} (\bm{x})}{2}}$, where $B_{3}\geq 0$ is also related to the privacy budget. The final prediction of the MRF is the majority vote over all the trees, which is the same as the one used in \cite{Breiman2001}:
\begin{equation}
\label{p3}
\overline{\hat{y}} = \overline{h^{(M)}(\bm{x})} = \arg \max_c \sum_{i} \1{h^{(i)}(\bm{x}) = c}.
\end{equation}

\section{Consistency and Privacy-Preservation Theorem}\label{CC}
In this section, we theoretically analyze the consistency and privacy-preservation of the MRF. All omitted proofs are shown in the \textbf{Appendix}.

\subsection{Consistency}
\subsubsection{Preliminaries}

\begin{definition}
When the dataset $\mathcal{D}$ is given, for a certain distribution of $(\bm{X},Y)$, a sequence of classifiers $\{h\}$ are consistent if the error probability $L$ satisfies
$$
\mathbb{E}(L) = \Pr(h(\bm{X},Z,\mathcal{D}) \neq Y)\rightarrow L^{*},
$$
where $L^{*}$ denotes the Bayes risk, $Z$ denotes the randomness involved in the construction of the tree, such as the selection of candidate features.

\end{definition}

\begin{lemma}
\label{vote}
The voting classifier $\overline{h^{(M)}}$ which takes the majority vote over $M$ copies of $h$ with different randomizing variables has consistency if those classifiers $\{h\}$ have consistency.
\end{lemma}

\begin{lemma}\label{l4}
Consider a partitioning classification rule building a prediction by a majority vote method in each leaf node. If the labels of the voting data have no effect on the structure of the classification rule, then $\E{L} \to L^*$ as $n \to \infty$ provided that
\begin{enumerate}
\item The diameter of $\mathcal{N}(\bm{X})\rightarrow 0$ as $n \to \infty$ in probability,
\item $|\mathcal{N}^E(\bm{X})| \to \infty$ as $n \to \infty$ in probability,
\end{enumerate}
where $\mathcal{N}(\bm{X})$ is the leaf containing $\bm{X}$, $|\mathcal{N}^E(\bm{X})|$ is the number of estimation points in $\mathcal{N}(\bm{X})$.
\label{devroye61-c}
\end{lemma}

Lemma \ref{vote} \cite{biau2008} states that the consistency of individual trees leads the consistency of the forest. Lemma \ref{devroye61-c} \cite{devroye2013} implies that the consistency of a tree can be ensured as $n\rightarrow\infty$, every hypercube at a leaf is sufficiently small but still contains infinite number of estimation points.

\subsubsection{Sketch Proof of the Consistency}

In general, the proof of consistency has three main steps: \textbf{(1)} each feature has a non-zero probability to be selected, \textbf{(2)} each split reduces the expected size of the split feature, and \textbf{(3)} split process can go on indefinitely. We first propose two lemmas for step \textbf{(1)} and \textbf{(2)} respectively, and then the consistency theorem of the MRF.

\begin{lemma}\label{P1}
  In the MRF, the probability that any given feature $A$ is selected to split at each node has lower bound $P_1>0$.
\end{lemma}

\begin{lemma}\label{P2}
  Suppose that features are all supported on $[0,1]$. In the MRF, once a split feature $A$ is selected, if this feature is divided into
  $N (N \geq 3)$ equal partitions $A^{(1)}, \cdots, A^{(N)}$ from small to large ($i.e.$, $A^{(i)}=\left[\frac{i-1}{N},\frac{i}{N} \right]$), for any split point $v$,
$$
\exists P_2\,(P_2>0), \mathrm{s.t.}\, \Pr\left(v\in \bigcup_{i=2}^{N-1} A^{(i)} | A\right)\geq P_2.
$$
\end{lemma}

Lemma \ref{P1} states that the MRF fulfills the first aforementioned requirement. Lemma \ref{P2} states that second condition is also met by showing that the specific split value has a large probability that it is not near the two endpoints of the feature interval.

\begin{theorem}
\label{th1}
Suppose that $\bm{X}$ is supported on $[0,1]^D$ and have non-zero density almost everywhere, the cumulative distribution function of the split points is right-continuous at 0 and left-continuous at 1. If $B_3 \rightarrow \infty$, MRF is consistent when $k \rightarrow \infty$ and $k/n \rightarrow 0$ as $n \rightarrow \infty$.
\end{theorem}

\subsection{Privacy-Preservation}\label{PP}
In this part, we prove that the MRF satisfies $\epsilon$-differential privacy based on two composition properties \cite{McSherry10}. Suppose we have a set of privacy mechanisms $\mathcal{M}=\{\mathcal{M}_{1}, \ldots, \mathcal{M}_{p}\}$ and each $\mathcal{M}_{i}$ provides $\epsilon_{i}$ privacy guarantee, then the sequential composition and parallel composition are described as follows: 

\begin{property}[Sequential Composition]\label{SC}
Suppose $\mathcal{M}=\{\mathcal{M}_{1}, \ldots, \mathcal{M}_{p}\}$ are sequentially performed on a dataset $\mathcal{D}$, then $\mathcal{M}$ will provide $(\sum_{i=1}^{p}\epsilon_{i})$-differential privacy.
\end{property}

\begin{property}[Parallel Composition]\label{PC}
Suppose $\mathcal{M}=\{\mathcal{M}_{1}, \ldots, \mathcal{M}_{p}\}$ are performed on a disjointed subsets of the entire dataset, $i.e.$, $\{\mathcal{D}_{1}, \ldots, \mathcal{D}_{p}\}$, respectively, then $\mathcal{M}$ will provide $(\max\{\epsilon_{i}\}_{i=1}^{p})$-differential privacy.
\end{property}

\begin{table*}[ht]
\caption{Accuracy (\%) of different RFs on benchmark UCI datasets.}
\label{ClassAcc}
\begin{center}
\small
\begin{sc}
\centering
\begin{tabular}{lcccc|cc}
\hline
{Dataset} & {\it Denil14} & BRF & CompRF-C & \textbf{MRF} &CompRF-I & {\it Breiman} \\ \hline
Zoo & 80.00 & 85.00 & 87.69 & $\textbf{90.64}^{\dagger}$  &93.39 & 87.38 \\
Hayes & 50.93 & 45.35 & 45.82 & $\textbf{79.46}^{\dagger}$ &46.04 & 77.58 \\
Echo & 78.46 & 88.46 & 89.63 & $\textbf{91.72}^{\dagger}$ &88.09 & 90.64 \\
Hepatitis & 62.17 & 63.46 & 62.50 & \textbf{64.32} &58.33 & 64.05 \\
Wdbc & 92.86 & 95.36 & 92.39 & \textbf{95.78} &94.26 & 96.01 \\
Transfusion & 72.97 & 77.70 & 76.53 & \textbf{78.53} &75.28 & $79.52^{\bullet}$ \\
Vehicle & 68.81 & 71.67 & 59.68 & \textbf{73.54} &64.86 & $74.70^{\bullet}$ \\
Mammo & 79.17 & 81.25 & 76.57 & \textbf{81.86} &78.72 & $82.31^{\bullet}$ \\
Messidor & 65.65 & 65.21 & 65.62 & \textbf{67.14} &66.14 & $68.35^{\bullet}$ \\
Website & 85.29 & 85.58 & 85.98 & $\textbf{89.80}^{\dagger}$ &88.34 & 88.12 \\
Banknote & 98.29 & 98.32 & 99.36 & $\textbf{99.49}^{\dagger}$ &99.02 & 99.12 \\
Cmc & 53.60 & 54.63 & 53.93 & $\textbf{56.12}^{\dagger}$ &54.61 & 55.11 \\
Yeast & 58.80 & 58.38 & 14.15 & \textbf{61.03} &10.66 & 61.71 \\
Car & 88.02 & 93.43 & 79.07 & \textbf{96.30} &92.17 & $97.42^{\bullet}$ \\
Image & 95.45 & 96.06 & 93.99 & \textbf{97.47} &96.16 & 97.71 \\
Chess & 61.32 & 97.12 & 94.77 & $\textbf{99.25}^{\dagger}$ &97.49 & 98.72 \\
Ads & 85.99 & 94.43 & 96.05 & \textbf{96.76} &96.44 & $97.59^{\bullet}$ \\
Wilt & 97.16 & 97.25 & 97.23 & \textbf{98.56} &98.27 & 98.10 \\
Wine-Quality & 57.31 & 56.68 & 53.22 & \textbf{60.56} &55.06 & $64.78^{\bullet}$ \\
Phishing & 94.35 & 94.47 & 95.44 & $\textbf{96.07}^{\dagger}$ &96.45 & 95.56 \\
Nursery & 93.42 & 93.52 & 91.01 & $\textbf{99.28}^{\dagger}$ &95.67 & 96.89 \\
Connect-4 & 66.19 & 76.75 & 72.82 & $\textbf{81.46}^{\dagger}$ &76.27 & 80.05 \\ \hline
Average Rank & 5.1 & 4.1 & 4.8 & \textbf{1.5} & 3.7 & 1.8 \\ \hline
\end{tabular}
\end{sc}
\end{center}
\vskip -0.1in
\end{table*}

\begin{lemma}\label{B1_DP}
  The impurity-based multinomial distribution $\mathcal{M}(\phi)$ of feature selection is essentially the exponential mechanism of differential privacy, and satisfies $B_{1}$-differential privacy. 
\end{lemma}

\begin{lemma}
  The impurity-based multinomial distribution $\mathcal{M}(\varphi)$ of split value selection is essentially the exponential mechanism of differential privacy, and satisfies $B_{2}$-differential privacy. 
\end{lemma}

\begin{lemma}
  The label selection of each leaf in a tree satisfies $B_{3}$-differential privacy. 
\end{lemma}

Based on the above properties and lemmas, we can obtain the following theorem: 

\begin{theorem}\label{privacy}
The proposed MRF satisfies $\epsilon$-differential privacy when the hyper-parameters $B_{1}$, $B_{2}$ and $B_{3}$ satisfy $B_{1}+B_{2}=\epsilon /(d\cdot t)$ and $B_{3}=\epsilon /t$, where $t$ is the number of trees and $d$ is the depth of a tree such that $d\leqslant \mathcal{O}(\frac{|\mathcal{D}^E|}{k})$.
\end{theorem}

\section{Experiments}
\subsection{Machine Learning}\label{settings}

\noindent \textbf{Dataset Selection}.
We conduct experiments on twenty-two UCI datasets used in previous consistent RF works \cite{denil2014,Wang2018,haghiri2018}. The specific description of used datasets is shown in the \textbf{Appendix}.

\noindent \textbf{Baselines}.
We select {\it Denil14} \cite{denil2014}, BRF \cite{Wang2018} and CompRF \cite{haghiri2018} as the baseline methods in the following evaluations. Those methods are the state-of-the-art consistent random forests variants. Specifically, we evaluate two different CompRF variants proposed in \cite{haghiri2018}, including consistent CompRF (CompRF-C) and inconsistent CompRF (CompRF-I). Besides, we provide the results of standard RF ({\it Breiman}) \cite{Breiman2001} as another important baseline for comparison.

\noindent \textbf{Training Setup}.
We carry out 10 times 10-fold cross validation to generate 100 forests for each method. All forests have $t=100$ trees, minimum leaf size $k=5$. Gini index is used as the impurity measure except for CompRF. In {\it Denil14}, BRF, CompRF, and RF, we set the size of the set of candidate features $\sqrt{D}$. The partition rate of all consistent RF variants is set to $1$. All settings stated above are based on those used in \cite{denil2014, Wang2018}. In MRF, we set $B_1=B_2=10$ and $B_3 \rightarrow \infty$ in all datasets, and the hyper-parameters of baseline methods are set according to their paper.

\begin{figure*}[!ht]
\centering
\subfigure[]{
\includegraphics[width=0.23\textwidth]{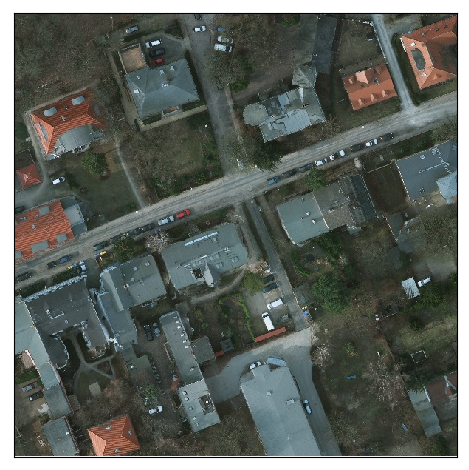}}
\subfigure[]{
\includegraphics[width=0.23\textwidth]{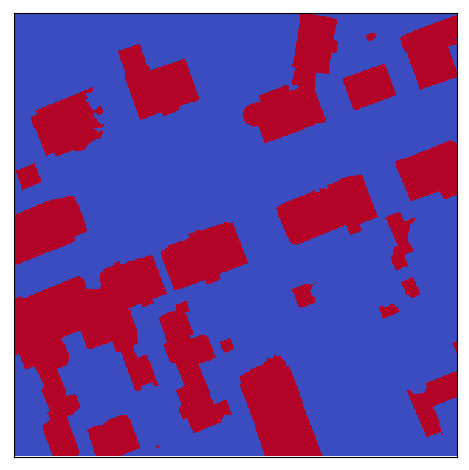}}
\subfigure[]{
\includegraphics[width=0.26\textwidth]{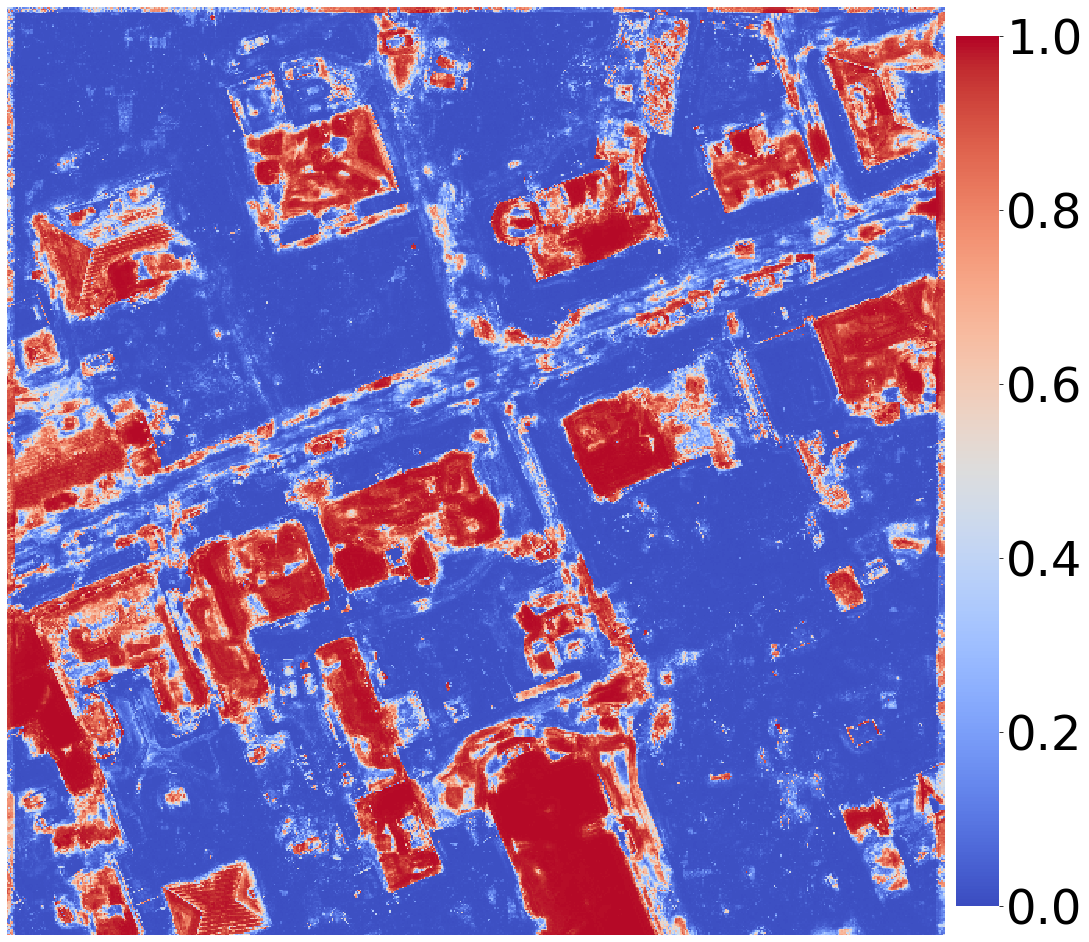}}
\caption{Visualization result of the proposed MRF. \textbf{(a)}: Aerial image; \textbf{(b)}: Groud-truth; \textbf{(c)}: The heat map of the prediction. The pixel is predicted as within the building area if and only if its color is red in the heat map.}
\label{fig:seg}
\vspace{-0.3em}
\end{figure*}

\begin{figure*}[!ht]
\centering
\subfigure[]{
\includegraphics[width=0.44\textwidth]{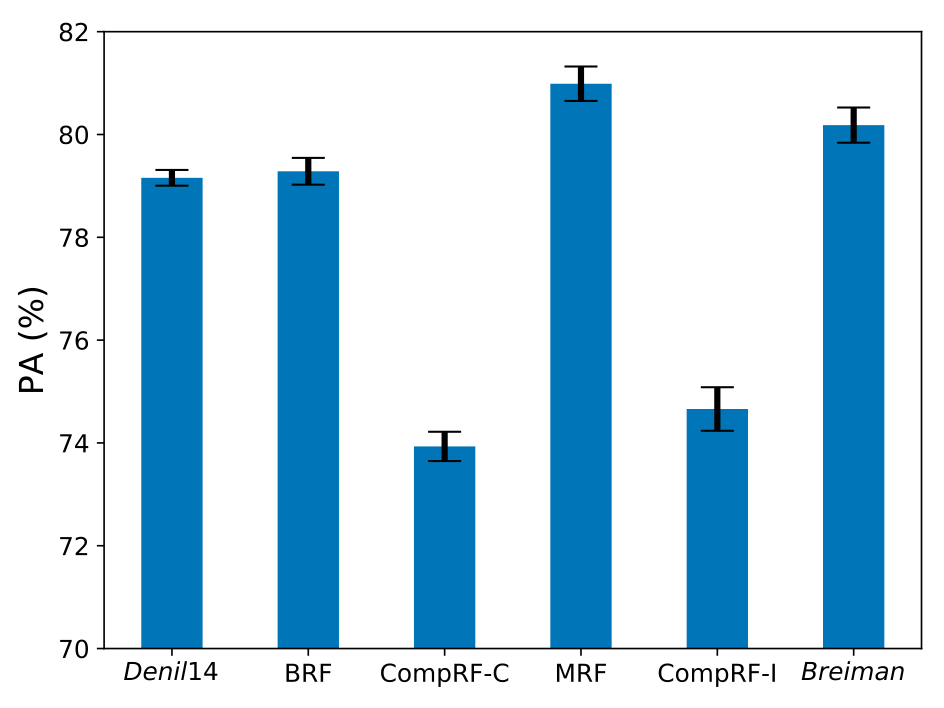}}
\subfigure[]{
\includegraphics[width=0.44\textwidth]{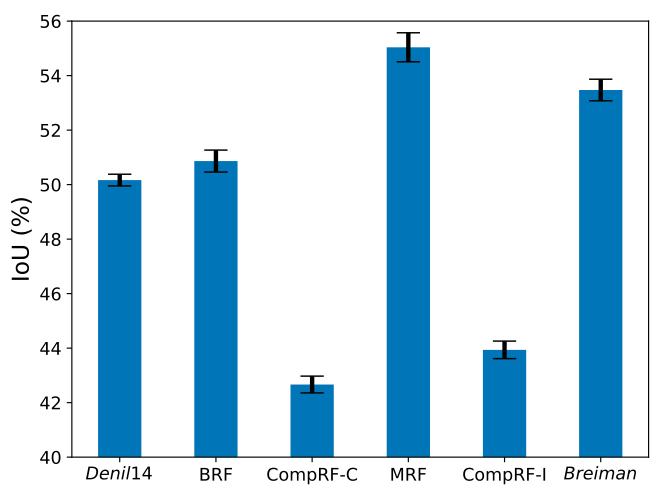}}
\caption{The pixel-wise accuracy (PA) and the intersection over union (IOU) of different methods. The standard deviation is indicated by the error bar.}
\label{fig:seg_result}
\vspace{-0.2em}
\end{figure*}

\noindent \textbf{Results}.
Table \ref{ClassAcc} shows the average test accuracy. Among the four consistent RF variants, the one with the highest accuracy is indicated in boldface. In addition, we carry out Wilcoxon's signed-rank test \cite{demvsar2006} to test for the difference between the results from the MRF and the standard RF at significance level 0.05. Those for which the MRF is significantly better than the standard RF are marked with "$\dagger$". Conversely, those for which RF is significantly better are marked with "$\bullet$". Moreover, the last line shows the average rank of different methods across all datasets.

As shown in Table \ref{ClassAcc}, MRF significantly exceeds all existing consistent RF variants. For example, MRF achieves more than $2\%$ improvement in most cases, compared with the current state-of-the-art method. Besides, the performance of the MRF even surpasses Breiman's original random forest in twelve of the datasets, and the advantage of the MRF is statistically significant in ten of them. To the best of our knowledge, this has never been achieved by any other consistent random forest methods. Note that we have not fine-tuned the hyper-parameters such as $B_1$, $B_2$ and $t$. The performance of the MRF might be further improved with the tuning of these parameters, which would bring additional computational complexity.

\subsection{Semantic Segmentation}

\noindent \textbf{Task Description}.
We treat the segmentation as a pixel-wise classification and build the dataset based on aerial images\footnote{\url{https://github.com/dgriffiths3/ml_segmentation}}. Each pixel of these images are labeled for one of two semantic classes: \emph{building} or \emph{not building}. Except for the RGB values of each pixel, we also construct some other widely used features. Specifically, we adopt local binary pattern \cite{ahonen2006face} with radius 24 to characterize texture, and calculate eight Haralick features \cite{Haralick1973} (including angular second moment, contrast, correlation, entropy, homogeneity, mean, variance, and standard deviation). We sample 10,000 pixels without replacement for training, and test the performance on the test image. To reduce the effect of randomness, we repeat the experiments 5 times with different training set assignments. Besides, all settings are the same as that of Section \ref{settings} unless otherwise specified.

\begin{figure*}[!ht]
\centering
\subfigure[ECHO]{
\label{fig1a}
\includegraphics[width=0.3\textwidth]{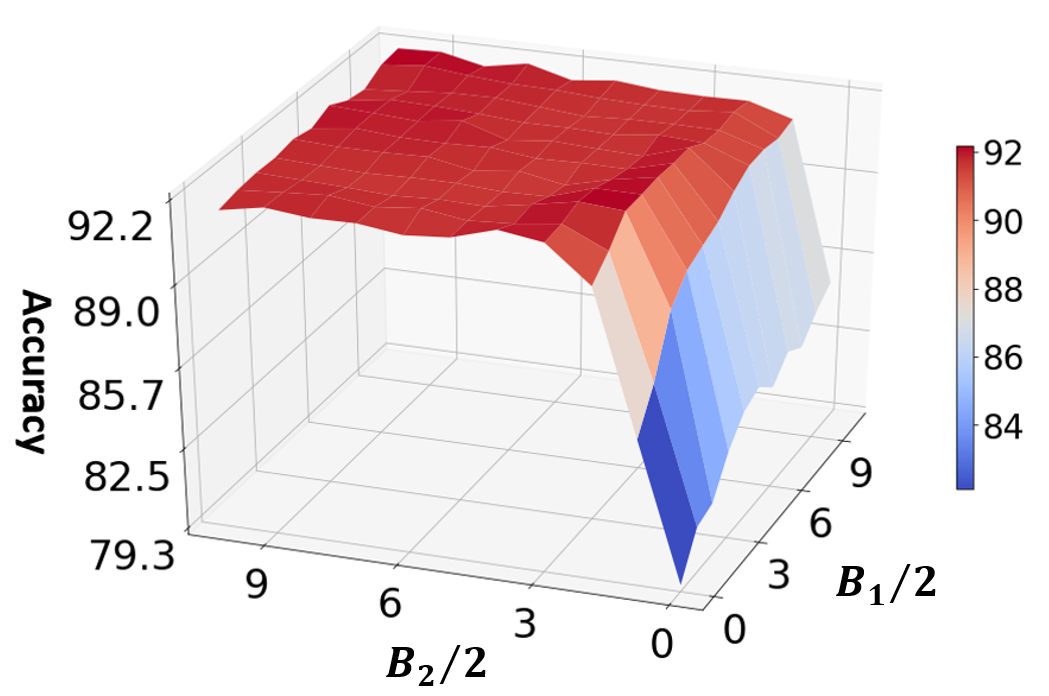}}
\subfigure[CMC]{
\label{fig1:subfig:b}
\includegraphics[width=0.3\textwidth]{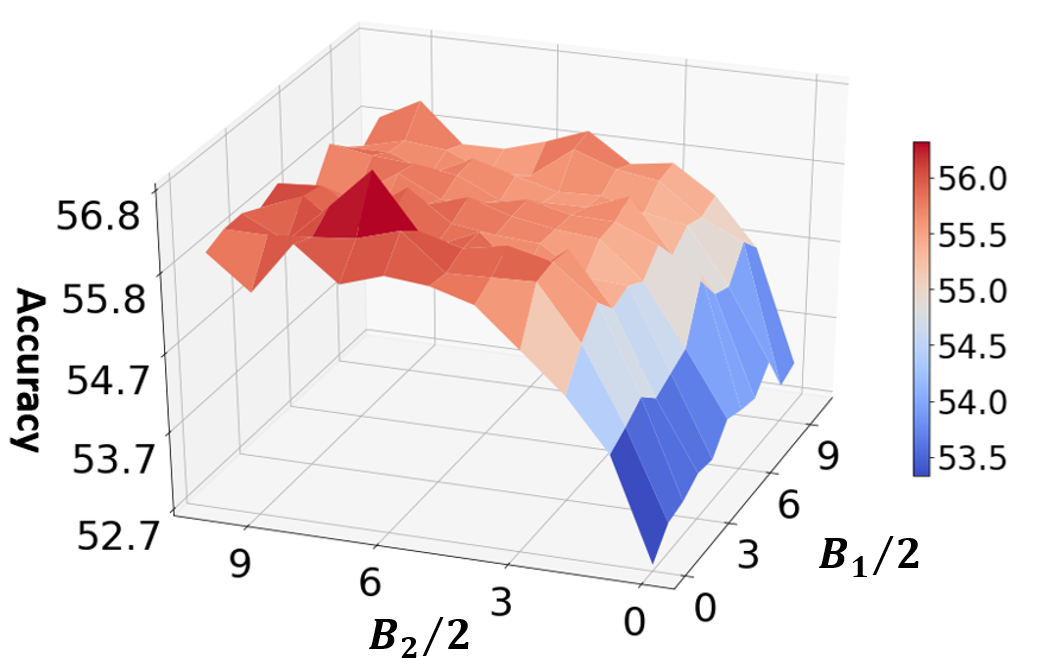}}
\subfigure[ADS]{
\label{fig1:subfig:c}
\includegraphics[width=0.3\textwidth]{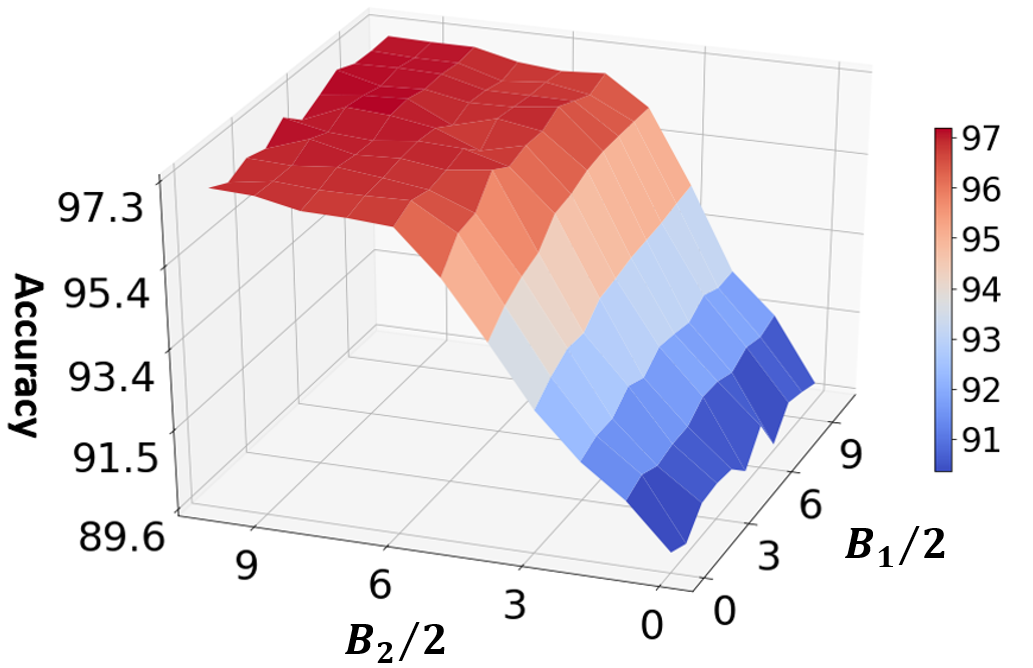}}
\subfigure[WDBC]{
\label{fig1:subfig:d}
\includegraphics[width=0.3\textwidth]{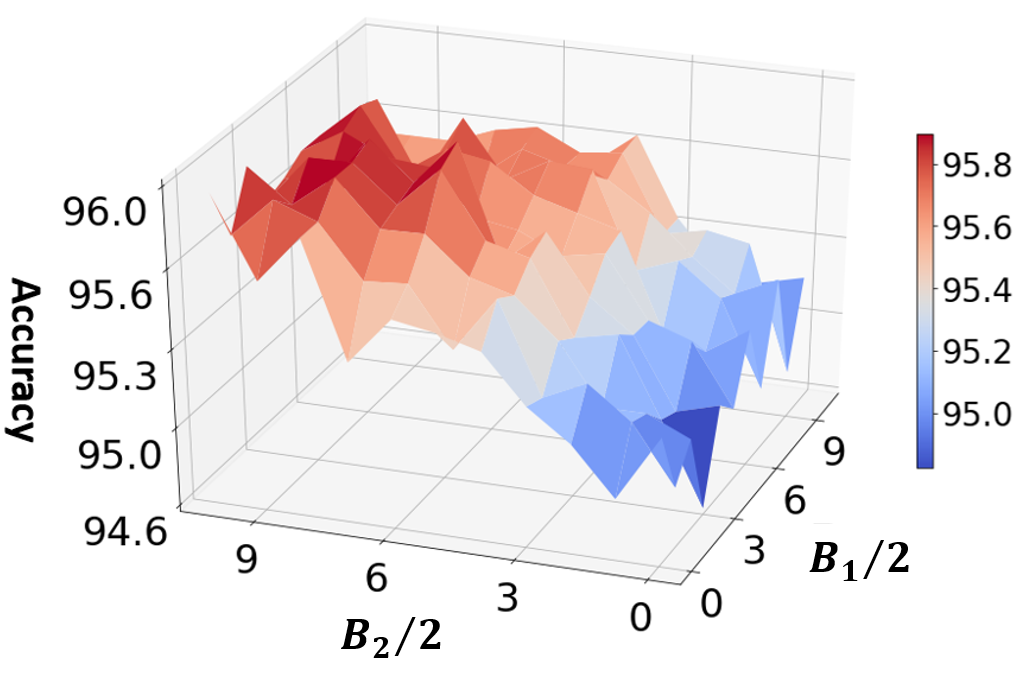}}
\subfigure[CAR]{
\label{fig1:subfig:e}
\includegraphics[width=0.3\textwidth]{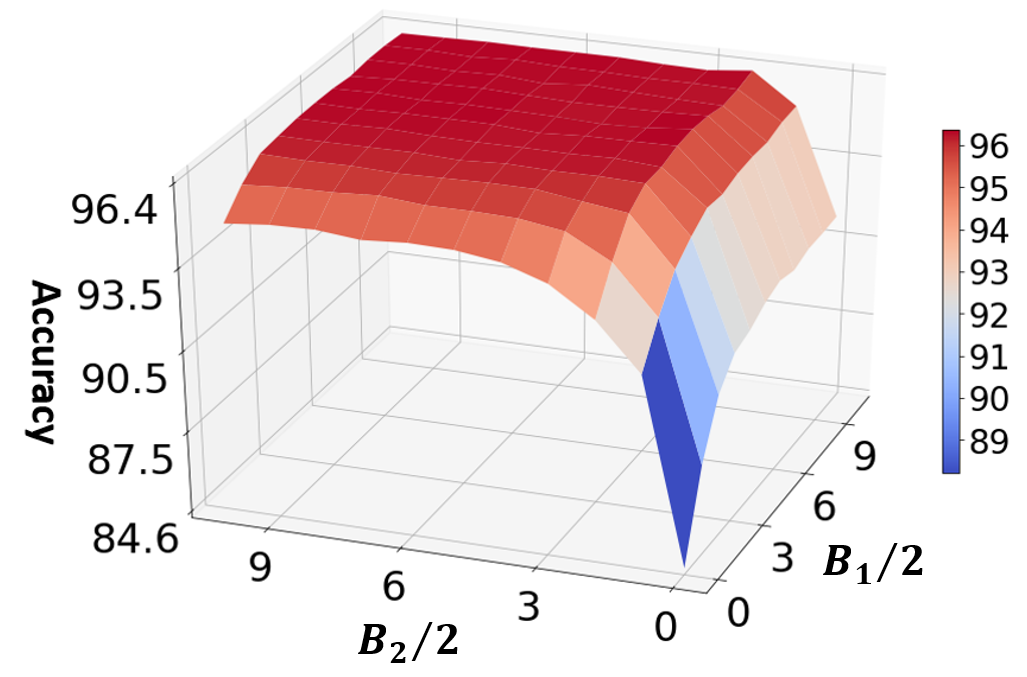}}
\subfigure[CONNECT-4]{
\label{fig1:subfig:f}
\includegraphics[width=0.3\textwidth]{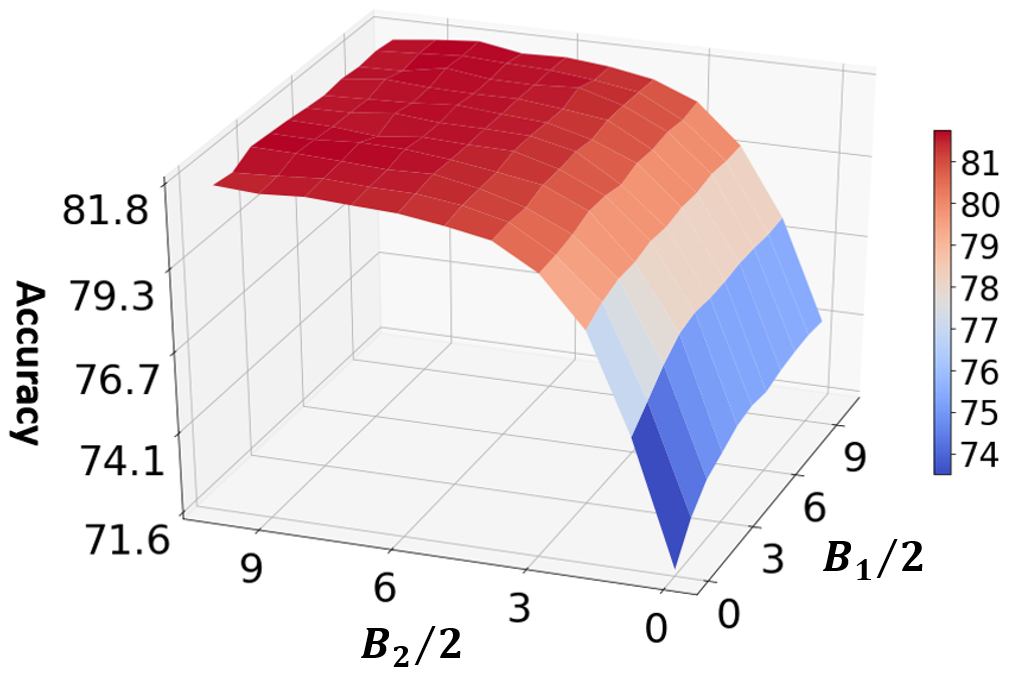}}
\caption{Accuracy (\%) of the MRF under different hyper-parameter values.}
\label{fig2} 

\end{figure*}

\noindent \textbf{Results}. We adopt two classical criteria to evaluate the performance of different models, including the pixel-wise accuracy (PA) and the intersection over union (IoU). As shown in Figure \ref{fig:seg_result}, the performance of MRF is better than that of RF. Compare with existing consistent RF, the improvement of MRF is more significant. We also visualize the segmentation results of MRF, as shown in Figure \ref{fig:seg}. Although the performance of MRF may not as good as some state-of-the-art deep learning based methods, it still achieves plausible results.

\subsection{The Effect of Hyper-parameters}
In this part, we evaluate the performance of the consistent MRF under different hyper-parameters $B_1$ and $B_2$. Specifically, we consider a range of $[0,20]$ for both $B_1$ and $B_2$, and other hyper-parameters are the same as those stated in Section \ref{settings}. Besides, the performance of each tree in MFR with respect to the privacy budget is shown in the \textbf{Appendix}.

Figure \ref{fig2} displays the results for six datasets representing small, medium and large datasets. It shows that the performance of the MRF is significantly improved as $B_2$ increases from zero, and it further becomes relatively stable when $B_2\ge 10$. Similarly, the performance also improves as $B_1$ increases from zero, but the effect is not obvious. When $B_2$ is too small, the resulting multinomial distributions would allow too much randomness, leading to the poor performance of the MRF. Besides, as shown in the figure, although the optimal values of $B_1$ and $B_2$ may depend on the specific characteristics of a dataset, such as the outcome scale and the dimension of the impurity decrease vector, at our default setting ($B_1=B_2=10$), the MRF achieves competitive performance in all datasets.

\section{Conclusion}
In this paper, we propose a new random forest framework, dubbed multinomial random forest (MRF), based on which we analyze its consistency and privacy-preservation property. In the MRF, we propose two impurity-based multinomial distributions for the selection of split feature and split value. Accordingly, the best split point has the highest probability to be chosen, while other candidate split points that are nearly as good as the best one will also have a good chance to be selected. This split process is more reasonable, compared with the greedy split criterion used in existing methods. Besides, we also introduce the exponential mechanism of differential privacy for selecting the label of a leaf to discuss the privacy-preservation of MRF. Experiments and comparisons demonstrate that the MRF remarkably surpasses existing consistent random forest variants, and its performance is on par with Breiman's random forest. It is by far the first random forest variant that is consistent and has comparable performance to the standard random forest.

\newpage

\section{Broader Impact}

Data privacy is critical for data security and consistency is also an important theoretical property in this big data era. As such, our work has positive impacts in general. 

Specifically, from the aspect of positive broader impacts, \textbf{(1)} our work theoretically analyzed the data privacy of the proposed RF framework, which allows the trade-off between performance and data privacy; \textbf{(2)} MRF is the first consistent RF variant whose performance is on par with that of standard RF, therefore it can be used as an alternative of RF and inspire further research in this area; \textbf{(3)} the proposed impurity-based random splitting process is empirically verified to be more effective compared with the standard greedy approach, whereas its principle is not theoretically discussed. It may further inspire the theoretical analysis of this method.

For the negative broader impact, the proposed method further verified the potential of data, which may further enhance the concerns about data privacy.

\bibliographystyle{unsrt}
\bibliography{example_paper}

\newpage

\begin{center}
    \begin{Large}
        \textbf{Appendix}
    \end{Large}
\end{center}

\begin{appendices}

\setcounter{lemma}{0}
\setcounter{theorem}{0}

\section{Omitted Proofs}
\label{proof}
\subsection{The Proof of Lemma 3}

\setcounter{lemma}{2}
\begin{lemma}
  In the MRF, the probability that any given feature $A$ is selected to split at each node has lower bound $P_1>0$.
\end{lemma}

\begin{proof}
Recall that the normalized impurity decrease vector $\hat{\bm{I}} \in [0,1]^D$. When $\hat{\bm{I}} = (1,0,\cdots,0)$, the probability that the first feature is selected for splitting is the largest, and when $\hat{\bm{I}} = (0,1,\cdots,1)$, the probability reaches smallest. Therefore 
\begin{equation*}
P_1 \triangleq \frac{1}{1+(D-1)e^{B_1}}\leq \Pr\left(v \in A\right)\leq \frac{e^{B_1}}{e^{B_1}+(D-1)}.
\end{equation*}
\end{proof}

\subsection{The Proof of Lemma 4}

\begin{lemma}
  Suppose that features are all supported on $[0,1]$. In the MRF, once a split feature $A$ is selected, if this feature is divided into
  $N (N \geq 3)$ equal partitions $A^{(1)}, \cdots, A^{(N)}$ from small to large ($i.e.$, $A^{(i)}=\left[\frac{i-1}{N},\frac{i}{N} \right]$), for any split point $v$,
$$
\exists P_2\,(P_2>0), \mathrm{s.t.}\, \Pr\left(v\in \bigcup_{i=2}^{N-1} A^{(i)} | A\right)\geq P_2.
$$
\end{lemma}

\begin{proof}
Suppose $m$ is the number of possible splitting values of feature $A$, similar to Lemma \ref{P1}, the probability that a value is selected for splitting satisfies the following restriction:
\begin{equation}
    \frac{1}{1+(m-1)e^{B_2}}\leq \Pr(v)\leq \frac{e^{B_2}}{e^{B_2}+(m-1)}.
\end{equation}
In this case,
\begin{equation}
\begin{array}{l}
\Pr\left(v\in \bigcup_{i=2}^{N-1} A^{(i)}|A\right)=\frac{\displaystyle \int_{\bigcup_{i=2}^{N-1} A^{(i)}}f(v)dv}{\displaystyle \int_{A}f(v)dv} \\
\geq \displaystyle \lim_{m \to +\infty}\left(\frac{\displaystyle \int_{\bigcup_{i=2}^{N-1} A^{(i)}}\frac{1}{1+(m-1)e^{B_2}}dv}{\displaystyle \int_{A}\frac{e^{B_2}}{e^{B_2}+(m-1)}dv}\right)\\
=\displaystyle \lim_{m \to +\infty}\frac{N-2}{N} \cdot \frac{e^{B_2}+(m-1)}{e^{B_2}+(m-1)e^{2B_2}}\\
=\frac{N-2}{N}e^{-2B_2} \triangleq P_2.
\end{array}
\end{equation}
\end{proof}

\subsection{The Proof of Theorem 1}
\begin{theorem}
Suppose that $\bm{X}$ is supported on $[0,1]^D$ and have non-zero density almost everywhere, the cumulative distribution function of the split points is right-continuous at 0 and left-continuous at 1. If $B_3 \rightarrow \infty$, MRF is consistent when $k \rightarrow \infty$ and $k/n \rightarrow 0$ as $n \rightarrow \infty$.
\end{theorem}

\begin{proof}
When $B_3 \rightarrow \infty$, the prediction in each node is based on majority vote, therefore it meets the prerequisite of Lemma 2. Accordingly, we can prove the consistency of MRF by showing that MRF meets two requirements in Lemma 2. 

Firstly, since MRF requires $|\mathcal{N}^E(\bm{X})| \ge k$ where $k \rightarrow \infty$ as $n \rightarrow \infty$, $|\mathcal{N}^E(\bm{X})| \rightarrow \infty$ when $n \rightarrow \infty$ is trivial. 

Let $V_m(a)$ denote the size of the $a$-th feature of $\mathcal{N}_m(\bm{X})$, where $\bm{X}$ falls into the node $\mathcal{N}_m(\bm{X})$ at $m$-th layer. To prove $diam(\mathcal{N}(\bm{X})) \rightarrow 0$ in probability, we only need to show that $\mathbb{E}(V_m(a)) \rightarrow 0$ for all $A_a \in \mathcal{A}$. For a given feature $A_a$, let $V^{*}_m(a)$ denote the largest size of this feature among all children of node $\mathcal{N}_{m-1}(\bm{X})$. By Lemma 4, we can obtain
\begin{equation}
\begin{aligned}
\mathbb{E}(V^{*}_m(a)) &\leq (1-P_2)V_{m-1}(a) + P_2 \frac{N-1}{N}V_{m-1}(a) = \left(1-\frac{1}{N}P_2\right)V_{m-1}(a).
\end{aligned}
\end{equation}

By Lemma 3, we can know
\begin{equation}
\begin{aligned}
\mathbb{E}(V_{m}(a)) &\leq (1-P_1)V_{m-1}(a) + P_1 \mathbb{E}(V^{*}_{m}(a)) = \left(1-\frac{1}{N}P_1P_2\right)V_{m-1}(a).
\end{aligned}
\end{equation}
Since $V_{0}(a)=1$,
\begin{equation}\label{c1}
 \mathbb{E}(V_{m}(a)) \leq \left(1-\frac{1}{N}P_1P_2\right)^m.
\end{equation}

Unlike the deterministic rule in the $Breiman$, the splitting point rule in our proposed MRF has randomness, therefore the final selected splitting point can be regarded as a random variable $W_i (i \in \{1,\cdots,m\})$, whose cumulative distribution function is denoted by $F_{W_i}$.

Let $M_1=\min(W_1,1-W_1)$ denotes the size of the root smallest child, we have
\begin{equation}
\begin{aligned}
\Pr(M_1\geq \sigma^{1/m}) 
&= \Pr(\sigma^{1/m} \leq W_1 \leq 1-\sigma^{1/m}) = F_{W_1}(1-\sigma^{1/m})-F_{W_1}(\sigma^{1/m}).
\end{aligned}
\end{equation}

WLOG, we normalize the values of all attributes to the range $[0, 1]$ for each node, then after $m$ splits, the smallest child at the $m$-th layer has the size at least $\sigma$ with the probability at least
\begin{equation}
    \prod_{i=1}^{m}\left(F_{W_i}(1-\sigma^{1/m})-F_{W_i}(\sigma^{1/m})\right).
\end{equation}

Since $F_{Wi}$ is right-continuous at 0 and left-continuous at 1, $\forall \alpha_1>0,\exists \sigma, \alpha>0$ s.t.
$$
    \prod_{i=1}^{m}\left(F_{W_i}(1-\sigma^{1/m})-F_{W_i}(\sigma^{1/m})\right)>(1-\alpha_1)^m>1-\alpha.
$$

Since the distribution of $\bm{X}$ has a non-zero density, each node has a positive measure with respect to $\mu_{\bm{X}}$. Defining
$$
p=\min_{\mathcal{N}: \mathrm{a\, node\, at} \, m-\mathrm{th}\, \mathrm{level}} \mu_X(\mathcal{N}),
$$
we know $p>0$ since the minimum is over finitely many nodes and each node contains a set of positive measure.

Suppose the data set with size $n$, the number of data points falling in the node $A$, where $A$ denotes the $m$-th level node with measure $p$, follows Binomial$(n, p)$. Note that this node $A$ is the one containing the smallest expected number of samples. WLOG, considering the $\mathrm{partition}\, \mathrm{rate} = 1$, the expectation number of estimation points in $A$ is $np/2$. From Chebyshev’s inequality, we have
\begin{equation}
\begin{array}{l}
\Pr\left(|\mathcal{N}^E(\bm{X})| < k\right)=
\Pr\left(|\mathcal{N}^E(\bm{X})|-\frac{np}{2} < k-\frac{np}{2}\right)\\
\leq \Pr\left(\left||\mathcal{N}^E(\bm{X})|-\frac{np}{2}\right|> \left|k-\frac{np}{2}\right|\right) \leq \frac{np(1-p)}{2|k-\frac{np}{2}|^2}\\
=\frac{p(1-p)}{2n|\frac{k}{n}-\frac{p}{2}|^2},
\end{array}
\end{equation}
where the first inequality holds since  $k-\frac{np}{2}$ is negative as $n \rightarrow \infty$ and the second one is by Chebyshev's inequality.

Since the right hand side goes to zero as $n \rightarrow \infty$, the node contains at least $k$ estimation points in probability. By the stopping condition, the tree will grow infinitely often in probability, $i.e.$, 
\begin{equation}\label{c2}
    m\rightarrow \infty.
\end{equation}

By (\ref{c1}) and (\ref{c2}), the theorem is proved.

\end{proof}

\subsection{The Proof of Lemma 5}
\begin{lemma}
  The impurity-based multinomial distribution of feature selection $\mathcal{M}(\phi)$ is essentially the exponential mechanism of differential privacy, and satisfies $B_{1}$-differential privacy. 
\end{lemma}
\begin{proof}\label{proof_B1}
As we all know, the softmax function is
\begin{equation*}
    f(\mathbf{x})_{j}=\frac{\exp{z_{j}}}{\sum_{i=1}^{D}\exp{z_{i}}}.
\end{equation*}
Obviously, the above formula is the same as the exponential mechanism (see the formula (2) in Definition 2). In what follows, we prove $\mathcal{M}(\phi)$ satisfies $B_{1}$-differential privacy.

 For any two neighboring datasets $\mathcal{D}^{S}$ and $\mathcal{D'}^{S}$, and any selected feature $A\in \mathcal{A}$, we can obtain 
\begin{equation}
\frac{\exp{\frac{B_{1}\hat{\bm{I}}(\mathcal{D}^{S}, A)}{2}}}{\exp{\frac{B_{1}\hat{\bm{I}}(\mathcal{D'}^{S}, A)}{2}}} =\exp{\frac{B_{1}\left(\hat{\bm{I}}(\mathcal{D}^{S}, A)-\hat{\bm{I}}(\mathcal{D'}^{S}, A)\right)}{2}}
\leqslant \exp {\frac{B_{1}}{2}},
\end{equation}
where the quality function $\hat{\bm{I}}(\mathcal{D}^{S}, A)$ represents the $a$-th item of the normalized feature vector $\hat{\bm{I}}$ based on the structure points dataset $\mathcal{D}^{S}$, and the corresponding sensitivity is $1$ through the normalized operation ($i.e.$, $\triangle \hat{\bm{I}}=\max_{\forall A, \mathcal{D}^{S}, \mathcal{D'}^{S}}|\hat{\bm{I}}(\mathcal{D}^{S}, A)-\hat{\bm{I}}(\mathcal{D'}^{S}, A)|=1 $).  Accordingly, the privacy of the split feature mechanism satisfies that for any output $A$ of $\mathcal{M}(\phi)$, we can obtain 

\begin{align*}
\small
& \frac{\Pr[\mathcal{M}(\phi, \mathcal{D}^{S})=A]}{\Pr[\mathcal{M}(\phi, \mathcal{D'}^{S})=A]}  =  \frac{\frac{\exp{\frac{B_{1}\hat{\bm{I}}(\mathcal{D}^{S}, A)}{2}}}{\sum_{A'\in \mathcal{A}}\exp{\frac{B_{1}\hat{\bm{I}}(\mathcal{D}^{S}, A')}{2}}}}{\frac{\exp{\frac{B_{1}\hat{\bm{I}}(\mathcal{D'}^{S}, A)}{2}}}{\sum_{A'\in \mathcal{A}}\exp{\frac{B_{1}\hat{\bm{I}}(\mathcal{D'}^{S}, A')}{2}}}}   
=  \frac{\exp{\frac{B_{1}\hat{\bm{I}}(\mathcal{D}^{S}, A)}{2}}}{\exp{\frac{B_{1}\hat{\bm{I}}(\mathcal{D'}^{S}, A)}{2}}}\cdot \frac{\sum_{A'\in \mathcal{A}}\exp{\frac{B_{1}\hat{\bm{I}}(\mathcal{D'}^{S}, A')}{2}}}{\sum_{A'\in \mathcal{A}}\exp{\frac{B_{1}\hat{\bm{I}}(\mathcal{D}^{S}, A')}{2}}}\\
&\leqslant \exp {\frac{B_{1}}{2}}\cdot \left(\frac{\sum_{A'\in \mathcal{A}}\exp{\frac{B_{1}}{2}}\exp{\frac{B_{1}\hat{\bm{I}}(\mathcal{D}^{S}, A')}{2}}}{\sum_{A'\in \mathcal{A}}\exp{\frac{B_{1}\hat{\bm{I}}(\mathcal{D}^{S}, A')}{2}}}\right) \\
& \leqslant \exp {\frac{B_{1}}{2}} \cdot \exp {\frac{B_{1}}{2}} \left(\frac{\sum_{A'\in \mathcal{A}}\exp{\frac{B_{1}\hat{\bm{I}}(\mathcal{D}^{S}, A')}{2}}}{\sum_{A'\in \mathcal{A}}\exp{\frac{B_{1}\hat{\bm{I}}(\mathcal{D}^{S}, A')}{2}}}\right) \\
& = \exp{B_{1}}.
\end{align*}
Therefore, for each layer of a tree, the privacy budget consumed by the split mechanism of features is $B_ {1}$. That is, $\mathcal{M}(\phi)$ satisfies $B_{1}$-differential privacy. 
\end{proof}

\subsection{The Proof of Lemma \ref{B2_DP}}
\begin{lemma}\label{B2_DP}
  The impurity-based multinomial distribution $\mathcal{M}(\varphi)$ of split value selection is essentially the exponential mechanism of differential privacy, and satisfies $B_{2}$-differential privacy. 
\end{lemma}
\begin{proof}
Similar to the proof of Lemma \ref{B1_DP}, the split value selection $\mathcal{M}(\varphi)$ is essentially the exponential mechanism of differential privacy. For any two neighboring datasets $\mathcal{D}^{S}$ and $\mathcal{D'}^{S}$, and any selected split value $a_{j}[i]\in a_{j}=\{a_{j}[1],\ldots,a_{j}[m]\}$ of the feature $A_{j}$, we can obtain 
$$
\small
\frac{\exp{\frac{B_{2}\hat{\bm{I}}^{(j)}(\mathcal{D}^{S}, a_{j}[i])}{2}}}{\exp{\frac{B_{2}\hat{\bm{I}}^{(j)}(\mathcal{D'}^{S}, a_{j}[i])}{2}}} = \exp{\frac{B_{2}\left(\hat{\bm{I}}^{(j)}(\mathcal{D}^{S}, a_{j}[i])-\hat{\bm{I}}^{(j)}(\mathcal{D'}^{S}, a_{j}[i])\right)}{2}} \leqslant \exp {\frac{B_{2}}{2}},
$$
where the quality function $\hat{\bm{I}}^{(j)}(\mathcal{D}^{S}, a_{j}[i])$ represents the $i$-th item of the normalized feature vector $\hat{\bm{I}}^{(j)}$ based on the structure points dataset $\mathcal{D}^{S}$, and the corresponding sensitivity is $1$ through the normalized operation.  

Accordingly, for any output $a_{j}[i]$ of $\mathcal{M}(\varphi)$, we can obtain

\begin{align*}
& \frac{\Pr\left[\mathcal{M}(\varphi, \mathcal{D}^{S})=a_{j}[i]\right]}{\Pr\left[\mathcal{M}(\varphi, \mathcal{D'}^{S})=a_{j}[i]\right]} =  \frac{\frac{\exp{\frac{B_{2}\hat{\bm{I}}^{(j)}(\mathcal{D}^{S}, a_{j}[i])}{2}}}{\sum_{a_{j}[k]\in A_{j}}\exp{\frac{B_{2}\hat{\bm{I}}^{(j)}(\mathcal{D}^{S}, a_{j}[k])}{2}}}}{\frac{\exp{\frac{B_{2}\hat{\bm{I}}^{(j)}(\mathcal{D'}^{S}, a_{j}[i])}{2}}}{\sum_{a_{j}[k]\in A_{j}}\exp{\frac{B_{2}\hat{\bm{I}}^{(j)}(\mathcal{D'}^{S}, a_{j}[k])}{2}}}}   \\
&=  \frac{\exp{\frac{B_{2}\hat{\bm{I}}^{(j)}(\mathcal{D}^{S}, a_{j}[i])}{2}}}{\exp{\frac{B_{2}\hat{\bm{I}}^{(j)}(\mathcal{D'}^{S}, a_{j}[i])}{2}}}\cdot \frac{\sum_{a_{j}[k]\in A_{j}}\exp{\frac{B_{2}\hat{\bm{I}}^{(j)}(\mathcal{D'}^{S}, a_{j}[k])}{2}}}{\sum_{a_{j}[k]\in A_{j}}\exp{\frac{B_{2}\hat{\bm{I}}^{(j)}(\mathcal{D}^{S}, a_{j}[k])}{2}}}\\
&\leqslant \exp {\frac{B_{2}}{2}}\cdot \left(\frac{\sum_{a_{j}[k]\in A_{j}}\exp{\frac{B_{2}}{2}}\exp{\frac{B_{2}\hat{\bm{I}}^{(j)}(\mathcal{D}^{S}, a_{j}[k])}{2}}}{\sum_{a_{j}[k]\in A_{j}}\exp{\frac{B_{2}\hat{\bm{I}}^{(j)}(\mathcal{D}^{S}, a_{j}[k])}{2}}}\right) \\
& \leqslant \exp {\frac{B_{2}}{2}} \exp {\frac{B_{2}}{2}} \left(\frac{\sum_{a_{j}[k]\in A_{j}}\exp{\frac{B_{2}\hat{\bm{I}}^{(j)}(\mathcal{D}^{S}, a_{j}[k])}{2}}}{\sum_{a_{j}[k]\in A_{j}}\exp{\frac{B_{2}\hat{\bm{I}}^{(j)}(\mathcal{D}^{S}, a_{j}[k])}{2}}}\right) \\
& = \exp{B_{2}}.
\end{align*}
Thus, the selection mechanism of split value for a specific feature satisfies $B_ {2}$-differential privacy.  
\end{proof}

\subsection{The Proof of Lemma 6}
\begin{lemma}\label{B3_DP}
  The label selection of each leaf in a tree satisfies $B_{3}$-differential privacy. 
\end{lemma}
\begin{proof}
For any two neighboring datasets $\mathcal{D}^{E}$ and $\mathcal{D'}^{E}$, and any predicted label $c\in \mathcal{K}=\{1,2,\ldots,K\}$ of $\bm x$ in a specific leaf, we can obtain
\begin{equation}
\frac{\exp{\frac{B_{3}\eta(\mathcal{D}^{E}, c)}{2}}}{\exp{\frac{B_{3}\eta(\mathcal{D'}^{E}, c)}{2}}} = \exp{\frac{B_{3}\left(\eta(\mathcal{D}^{E}, c)-\eta(\mathcal{D'}^{E}, c)\right)}{2}} \leqslant \exp {\frac{B_{3}}{2}},
\end{equation}
where the quality function $\eta(\mathcal{D}^{E}, c)$ is equivalent to $\eta^{c}(\bm{x})$ and represents the predict label $c$ of the sample $\bm{x}$ based on the estimation points dataset $\mathcal{D}^{E}$. It is worth noting that $\eta(\mathcal{D}^{E}, c)$ is the empirical probability that sample $\bm{x}$ has the label $c$, and thus the corresponding sensitive is $1$. 
Then, for any output $c\in\{1,2,\ldots,K\}$ of $h(\bm x)$, we can obtain
\begin{align*}
& \frac{\Pr[h(\bm x, \mathcal{D}^{E})=c]}{\Pr[h(\bm x, \mathcal{D'}^{E})=c]}  =  \frac{\frac{\exp{\frac{B_{3}\eta(\mathcal{D}^{E}, c)}{2}}}{\sum_{c'\in \mathcal{K}}\exp{\frac{B_{3}\eta(\mathcal{D}^{E}, c')}{2}}}}{\frac{\exp{\frac{B_{3}\eta(\mathcal{D'}^{E}, c)}{2}}}{\sum_{c'\in \mathcal{K}}\exp{\frac{B_{3}\eta(\mathcal{D'}^{E}, c')}{2}}}} \\ & =  \frac{\exp{\frac{B_{3}\eta(\mathcal{D}^{E}, c)}{2}}}{\exp{\frac{B_{3}\eta(\mathcal{D'}^{E}, c)}{2}}}\cdot \frac{\sum_{c'\in \mathcal{K}}\exp{\frac{B_{3}\eta(\mathcal{D'}^{E}, c')}{2}}}{\sum_{c'\in \mathcal{K}}\exp{\frac{B_{3}\eta(\mathcal{D}^{E}, c')}{2}}}\\
&\leqslant \exp {\frac{B_{3}}{2}}\cdot \left(\frac{\sum_{c'\in \mathcal{K}}\exp{\frac{B_{3}}{2}}\exp{\frac{B_{3}\eta(\mathcal{D}^{E}, c')}{2}}}{\sum_{c'\in \mathcal{K}}\exp{\frac{B_{3}\eta(\mathcal{D}^{E}, c')}{2}}}\right) \\
& \leqslant \exp {\frac{B_{3}}{2}} \cdot \exp {\frac{B_{3}}{2}} \left(\frac{\sum_{c'\in \mathcal{K}}\exp{\frac{B_{3}\eta(\mathcal{D}^{E}, c')}{2}}}{\sum_{c'\in \mathcal{K}}\exp{\frac{B_{3}\eta(\mathcal{D}^{E}, c')}{2}}}\right) \\
& = \exp{B_{3}}.
\end{align*}

\begin{algorithm}[!ht]
   \caption{Decision Tree Training in MRF: $MTree()$ }
   \label{alg-dt}
\begin{algorithmic}[1]
   \STATE {\bfseries Input:} Structure points $\mathcal{D}^S$, estimation points $\mathcal{D}^E$ and hyper-parameters $k$, $B_1$, $B_2$.
   \STATE {\bfseries Output:} A decision tree $T$ in MRF.
        \IF{$|\mathcal{D}^E|>k$}        
        \STATE Calculate the impurity decrease of all possible split points $v_{ij}$.
        \STATE Select the largest impurity decrease of each feature to create
        a vector $\bm{I}$, calculate the normalized vector $\hat{\bm{I}}$, and compute the probabilities $\bm{\phi} = \mathrm{softmax}(\frac{B_1}{2}\hat{\bm{I}})$.
        \STATE Select a split feature randomly according to the multinomial distribution $M(\bm{\phi})$.
        \STATE Calculate the normalized vector $\hat{\bm{I}}^{(j)}$ for the selected split feature $f_{j}$, and compute the probabilities $\bm{\varphi} = \mathrm{softmax}(\frac{B_2}{2}\hat{\bm{I}}^{(j)})$.
        \STATE Select a split value randomly according to the multinomial distribution $M(\bm{\varphi})$. $\mathcal{D}^S$ and $\mathcal{D}^E$ are correspondingly split into two disjoint subsets $\mathcal{D}^{S_l}, \mathcal{D}^{S_r}$ and $\mathcal{D}^{E_l}, \mathcal{D}^{E_r}$, respectively.
        \STATE $T.leftchild \leftarrow MTree(\mathcal{D}^{S_l}, \mathcal{D}^{E_l},k,B_1,B_2)$
        \STATE $T.rightchild \leftarrow MTree(\mathcal{D}^{S_r}, \mathcal{D}^{E_r},k,B_1,B_2)$
        \ENDIF
\STATE {\bfseries Return:} A decision tree $T$ in MRF
\end{algorithmic}
\end{algorithm}

\begin{figure}[ht]
 \centering
 \includegraphics[width=0.6\textwidth]{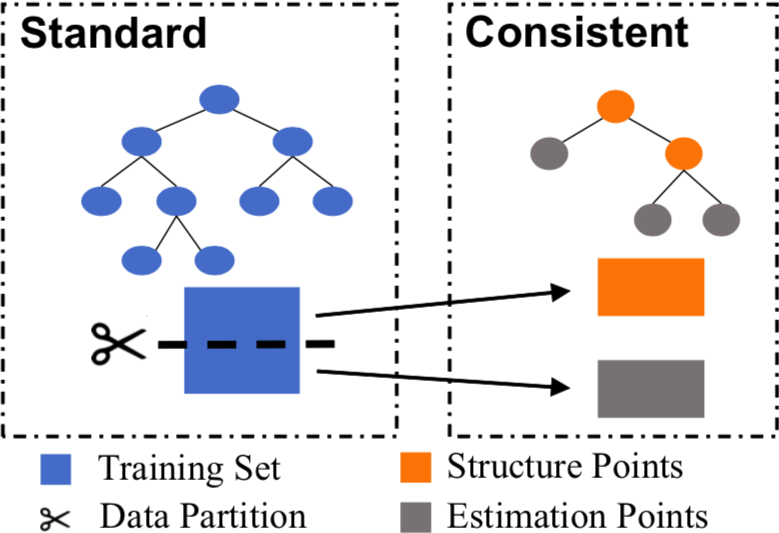}
 \caption{An illustration of the data partition.}
 \label{partition}
\end{figure}

Since each leaf divides the dataset $\mathcal{D}^E$ into disjoint subsets, according to Property 2, the prediction mechanism of the each tree satisfies $B_ {3}$-differential privacy.  
\end{proof}

\subsection{The Proof of Theorem \ref{privacy}}
\begin{theorem}
The proposed MRF satisfies $\epsilon$-differential privacy when the hyper-parameters $B_{1}$, $B_{2}$ and $B_{3}$ satisfy $B_{1}+B_{2}=\epsilon /(d\cdot t)$ and $B_{3}=\epsilon /t$, where $t$ is the number of trees in the MRF and $d$ is the depth of a tree such that $d\leqslant \mathcal{O}(\frac{|\mathcal{D}^E|}{k})$.
\end{theorem}
\begin{proof}
Based on Property 1 together with Lemma \ref{B1_DP} and Lemma \ref{B2_DP} , the privacy budget consumed for each layer of a tree is $B_ {1}+B_{2}=\epsilon /(d\cdot t)$. Since the depth of a tree is $d$, the total privacy budget consumed by the generation of tree structure is $d(B_{1}+B_{2})=\epsilon/ t$. Since the datasets $\mathcal{D}^S$ and $\mathcal{D}^E$ are disjoint, according to Property 2, the total privacy budget of a tree is $\max\{d(B_{1}+B_{2}), B_{3}\}=\epsilon/t$. 

As a result, the consumed privacy budget of the MRF containing $t$ trees is $\frac{\epsilon}{t}\cdot t=\epsilon$, which implies that the MRF satisfies $\epsilon$-differential privacy.
\end{proof}

\section{More Details about the Training Process of MRF}
We provide more details about the training process of MRF. The illustration of partition process, and the pseudo code of training process is shown in Figure \ref{partition} and Algorithm \ref{alg-dt}, respectively.

\begin{figure*}[ht]
 \centering
 \includegraphics[width=0.8\textwidth]{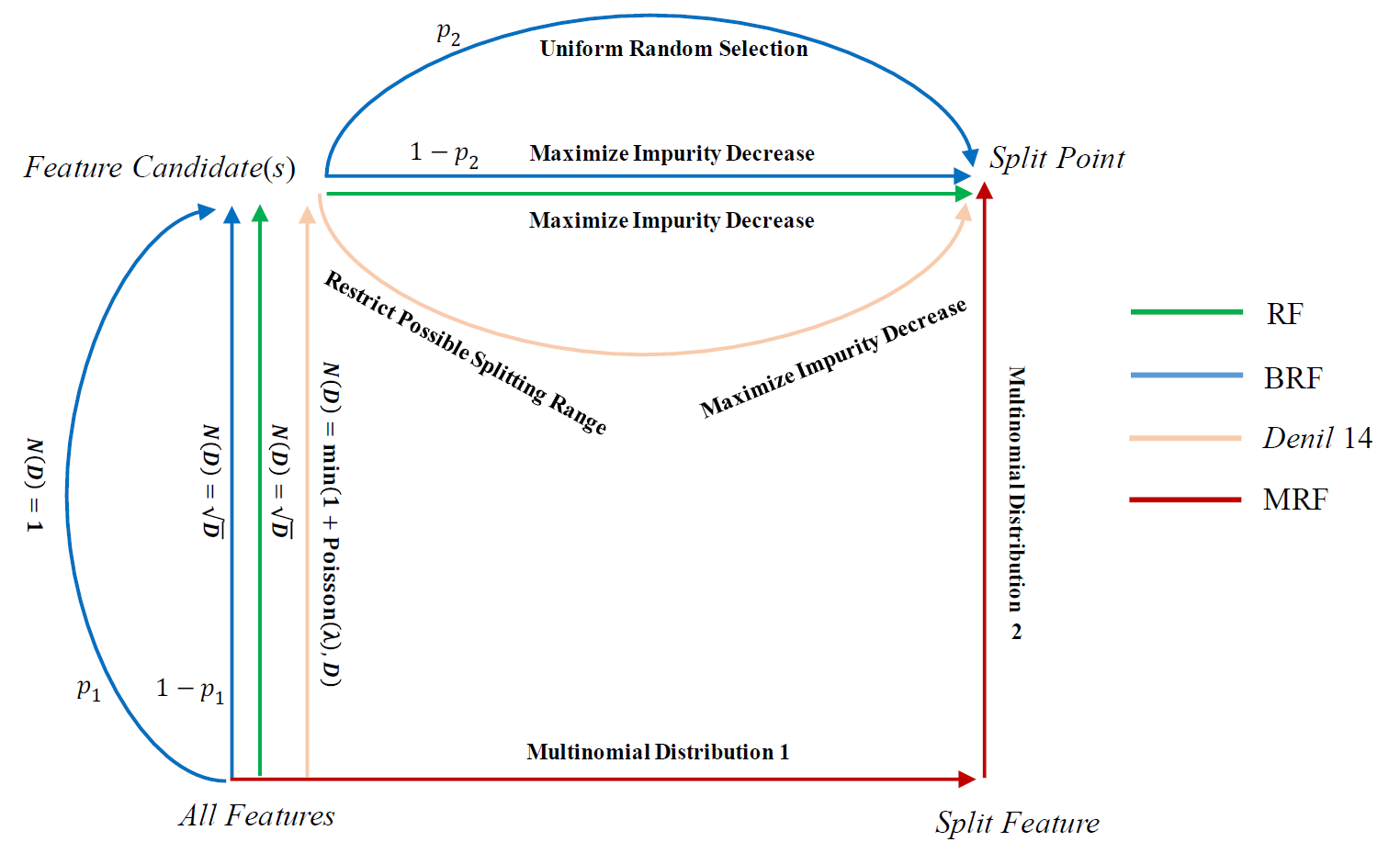}
 \caption{A diagram showing the differences among the standard RF, {\it Denil14}, BRF, and MRF. For {\it Denil14} and BRF, they choose split point in a greedy way mostly, while holding a small or even negligible probability in selecting anther point randomly. Their split process is very similar to the one of standard random forests. In contrast, in MRF, two impurity-based multinomial distributions are used for randomly selecting a split feature and a specific split point respectively. All randomness in consistent RF variants is aiming to fulfill the consistency, whereas the one in MRF is more reasonable.}
 \label{difference}
\end{figure*}

\section{The Discussion of Parameter Settings in MRF}\label{parameter_setting}
Firstly, we discuss the parameters settings from the aspect of privacy-preservation by setting $B_{1}$, $B_{2}$ and $B_{3}$ to ensure that MRF satisfies $\epsilon$-differential privacy. 

Suppose the number of trees and the depth of each tree are $t$ and $d$, respectively. To fulfill the $\epsilon$-differential privacy, we can evenly allocate the total privacy budget $\epsilon$ to each tree, $i.e.$, the privacy budget of each tree is $\epsilon/ t$. For each tree, the upper bound of depth $d$ is approximately $\mathcal{O}(\frac{|\mathcal{D}^E|}{k})$, $i.e.$, $d\leqslant \mathcal{O}(\frac{|\mathcal{D}^E|}{k})$. Accordingly, we can directly set $d=\frac{|\mathcal{D}^E|}{k}$ and evenly allocate the total privacy budget $\epsilon/ t$ to each layer, $i.e.$, the privacy budget of each layer is $\epsilon/(d\cdot t)$. As such, we can set $B_{1}$, $B_{2}$ and $B_{3}$ to $B_{1}+B_{2}=\epsilon/(d\cdot t)$ and $B_{3}=\epsilon/ t$ to ensure that MRF satisfies $\epsilon$-differential privacy. 

From another perspective, the hyper-parameters play a role in regulating the relative probabilities under which the 'best' candidate is selected. Specifically, the larger $B_1$, the smaller noise will be added in $\phi$, and thus the split feature with the largest impurity decrease $\hat{\bm{I}}_{i}$ ($i=1,2,\ldots,D$) has a higher probability of being selected. Similarly, $B_2$ and $B_{3}$ control the noises added to the split value selection and the label selection, respectively. In addition, if $ B_3 \rightarrow \infty$, regardless of the training set partitioning, when $B_1=B_2=0$, all features and split values have the same probability to be selected, and therefore the MRF would become a completely random forest \cite{Liu2008}; when $B_1, B_2 \rightarrow \infty$, there is no added noise. In this case, the MRF would become the Breiman's original RF, whose set of candidate features always contains all the features.

\section{Comparison among Different RFs}\label{compare}
In this section, we first explain the differences between the MRF and the
Breiman's random forest (Breiman2001), and then describe the
similarities and differences between the MRF and two other recently proposed consistent RF variants, {\it Denil14} \cite{denil2014} and BRF
\cite{Wang2018}. The differences are summarized in the Figure \ref{difference}.

In the standard RF, node split has two steps: 1) a restrictive step in
which only a subset of randomly selected features are considered, and 2) a
greedy step in which the best combination of split feature and split
value is chosen so that the impurity decrease is maximized after the split.
Both steps may have limitations.  The restriction in the first step is to
increase the diversity of the trees but at the cost of missing an informative
feature.  The greedy approach in the second step may be too exclusive,
especially when there are several choices that are not significantly
different from each other.  The MRF improves on both steps.  First, we
introduce an approach to allow a more informed selection of split
features.  We compute the largest possible impurity decrease for every
feature to represent the potentials of the features, and then use them as a
basis to construct a multinomial distribution for feature selection.  Second,
once a split feature is selected, we introduce randomness in the
selection of a split value for the feature.  We consider all possible
values and use their corresponding impurity decreases as a basis to construct
a second multinomial distribution for split point selection.  This way,
the best split point in the standard RF still has the highest probability
to be selected although it is not always selected as in the greedy approach.
These new procedures allow us to achieve a certain degree of randomness
necessary for the proof of consistency and to still maintain the overall
quality of the trees.  In the MRF, the diversity of the trees is achieved
through the randomness in both steps (plus data partitioning); in contrast, in the standard RF, the diversity is
achieved through the restrictive measure in the first step (plus bootstrap
sampling).

All three RF variants -- {\it Denil14}, BRF and MRF -- have the same
procedure of partitioning the training set into two subsets to be used
separately for tree construction and for prediction. This procedure is
necessary for the proof of consistency, which is in contrast to the standard RF, where a bootstrap dataset is used for both tree construction and prediction.

For feature selection, all three RF variants have a procedure to ensure that
every feature has a chance to be splitted, which is necessary
for proving consistency. However, the MRF is completely different from the
other two methods in this aspect. The other two methods introduce a simple
random distribution (a Poisson distribution in {\it Denil14} and a Bernoulli
distribution in BRF), but they are restrictive as they use a random subset of
features as in the standard RF.  In the MRF, the randomness is built into an
impurity-based multinomial distribution defined over all the features.

When determining the split point, all three RF variants have some level
of randomness to ensure that every split point has a chance to be used,
which is again necessary for proving consistency.  However, all the methods
except the MRF involve a greedy step, in which the final selection is
deterministic.  That is, whenever there is a pool of candidate split
points, those methods always select the ``best'' one that has the largest
impurity decrease.  Specifically, in {\it Denil14}, a pool of randomly
selected candidate split points is created and then searched for the
``best'' one.  In BRF, a Bernoulli distribution is used so that with a small
probability, the split is based on a randomly selected point, and with a
large probability, the choice is the same as in the standard RF, that is, the
``best'' among all possible split points. This greedy approach of
choosing the ``best'' split point at every node may be too exclusive.
This is especially true when there are several candidate split points
that are not significantly different from each other.  In this case, it may
not be a good strategy to choose whichever one that happens to have the
largest impurity decrease while ignoring all others that are nearly as good. This issue is solved by the non-greedy approach in the MRF, in which a
split point is chosen randomly according to another impurity-based
multinomial distribution. This way, the ``best'' split point has the
highest probability to be chosen, but other candidate split points that
are nearly as good as the ``best'' one will also have a good chance to be
selected.  This flexibility in split point selection in the MRF may
partially explain its good performance to be shown in the next section.

In addition, the prediction of the MRF trees are also different from those of {\it Denil14}, BRF and RF. Specifically the predicted label is randomly selected with a probability proportional to $\exp {\frac{B_{3}\eta^{(c)} (\bm{x})}{2}}$ in MRF, whereas the prediction of other three methods are deterministic. This modification is to ensure that the prediction process also satisfies the exponential mechanism, therefore the privacy-preservation can be analyzed under the differential privacy framework.

\section{The Description of UCI Benchmark Datasets}
We conduct machine learning experiments on multiple UCI datasets used in previous consistent RF works \cite{denil2014,Wang2018,haghiri2018}. These datasets cover a wide range of sample size and feature dimensions, and therefore they are representative for evaluating the performance of different algorithms. The description of used datasets is shown in Table \ref{Dataset}.

\begin{table}[ht]
\caption{The description of UCI benchmark datasets.}
\label{Dataset}
\begin{center}
\begin{small}
\begin{sc}
\begin{tabular}{lccc}
\hline
Dataset & Samples & Features & Classes\\
\hline
Zoo     &101 & 17 & 7\\
Hayes   &132 & 5 & 3\\
Echo    &132 & 12 & 2\\
Hepatitis &155 &19 &2\\
Wdbc    &569& 39 & 2\\
Transfusion &748 & 5 & 2\\
Vehicle    &946 & 18 & 4\\
Mammo  &961 &6 &2\\
Messidor &1151 &19 &2\\
Website &1353 &9 &3\\
Banknote &1372 &4 &2\\
Cmc    & 1473 & 9 & 3\\
Yeast &1484 &8 &10\\
Car    & 1728 & 6 & 4\\
Image    & 2310 & 19 & 7\\
Chess    & 3196 & 36 & 2\\
Ads    & 3729 & 1558 & 2\\
Wilt   &4839 &5 &2\\
Wine-Quality &4898 &11 &7\\
Phishing &11055& 31& 2\\
Nursery & 12960 & 9 & 5\\
Connect-4    & 67557 &42 & 3\\
\hline
\end{tabular}
\end{sc}
\end{small}
\end{center}
\vskip -0.2in
\end{table}

\begin{table}[ht]
\caption{Test accuracy ($\%$) of a tree in MFR in terms of different privacy budgets $B_1$, $B_2$, and $B_3$.}
\small
\label{tab:DP}
\begin{center}
    \begin{tabular}{|c|c|c|c|c|c|}
    \hline
     \multirow{2}{*}{$B_{3}$}  & \multirow{2}{*}{$B_{1}$} & \multirow{2}{*}{$B_{2}$} & \multicolumn{3}{|c|}{Datasets}  \\
     \cline{4-6}
         & & & WDBC & CMC & CONNECT-4 \\
     \hline
     1 & 0.05 & 0.05& 94.71 & 52.58 & 72.20 \\
     \hline
     5 & 0.25 & 0.25 & 94.95 & 52.79 & 74.12 \\
     \hline
     10 & 0.5 & 0.5 & 94.93 & 52.94 & 74.84 \\
     \hline
     20 & 1 & 1 & 95.21 & 53.25 & 76.84\\
     \hline
    \end{tabular}
\end{center}
\end{table}

\section{Performance of Differential Privacy}
In this section, we simulate the performance of each tree in our MRF for different privacy budgets. We conduct the experiments in WDBC, CMC and CONNECT-4 dataset, which is the representative of small, medium and large datasets, respectively. Speficially, according to Theorem \ref{privacy}, each tree in the MRF satisfies $B_{3}$-differential privacy, and the privacy budget consumed for each layer of a tree is $B_{1}+B_{2}$, which satisfies $B_{1}+B_{2}=B_{3}/d$. Therefore, Table \ref{tab:DP} presents the performance changes with respect to $B_3$. In the experiments, we set $B_{3}=1$, $5$, $10$, and $20$ respectively. Since we focus on the trade-off between the accuracy and privacy, we can simply set $B_{1}=B_{2}$. Besides,  we observe that the depth of each tree in MRF constructed based on selected datasets is no more than $10$, therefore we directly set $d=10$.

From the table, we can see that when $B_{3}$ increases, the performance of each tree increases, which meets the changing trend of differential privacy. Specifically, when the privacy budget is relatively low, the added noise is relatively high, which results in reduced performance. On the contrary, when the privacy budget is relatively high, the added noise is relatively low, and thus the corresponding performance will increase. Besides, we can observe that when $B_{3}$ decreases, the performance is hardly reduced. Thus, in practice, we can set $B_{3}$ to be relatively low, to ensure better security without reducing performance.

\begin{figure*}[ht]
\centering
\subfigure[ECHO]{
\label{fig3a}
\includegraphics[width=0.28\textwidth]{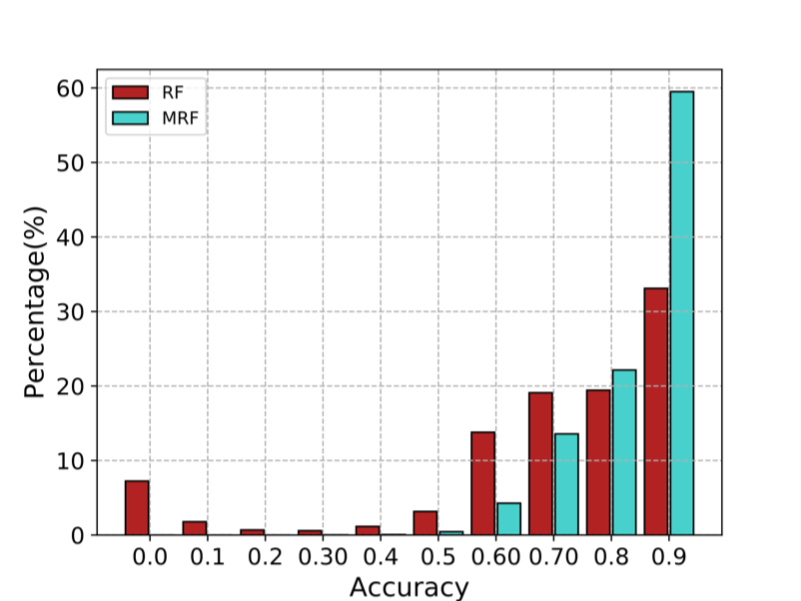}}
\subfigure[CMC]{
\label{fig3:subfig:b}
\includegraphics[width=0.28\textwidth]{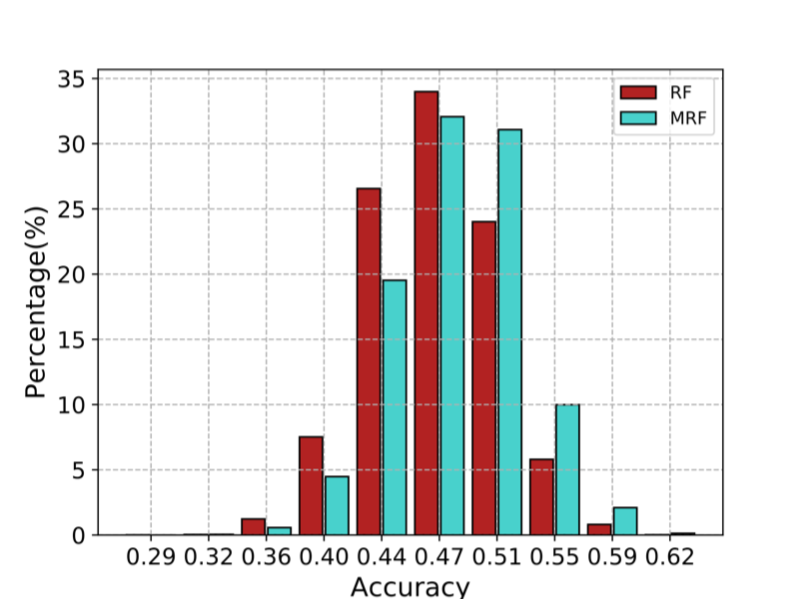}}
\subfigure[ADS]{
\label{fig3:subfig:c}
\includegraphics[width=0.28\textwidth]{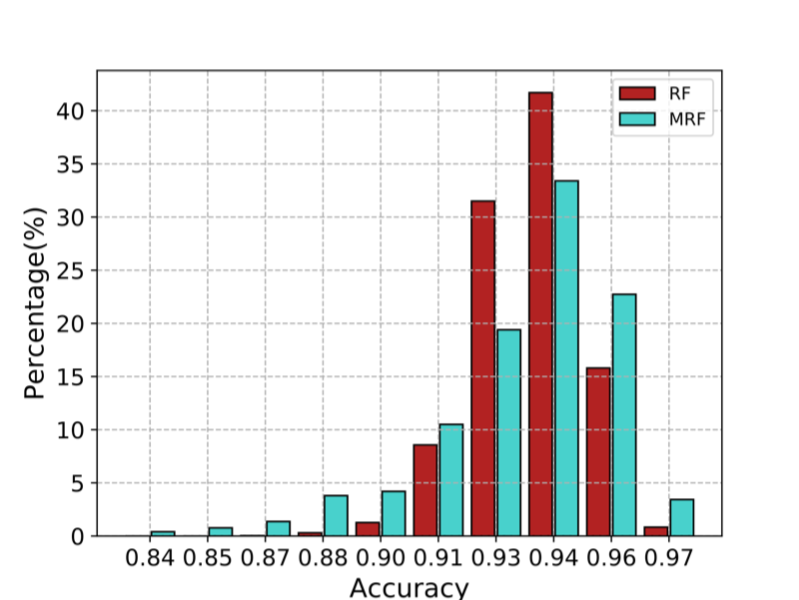}}
\subfigure[WDBC]{
\label{fig3:subfig:d}
\includegraphics[width=0.28\textwidth]{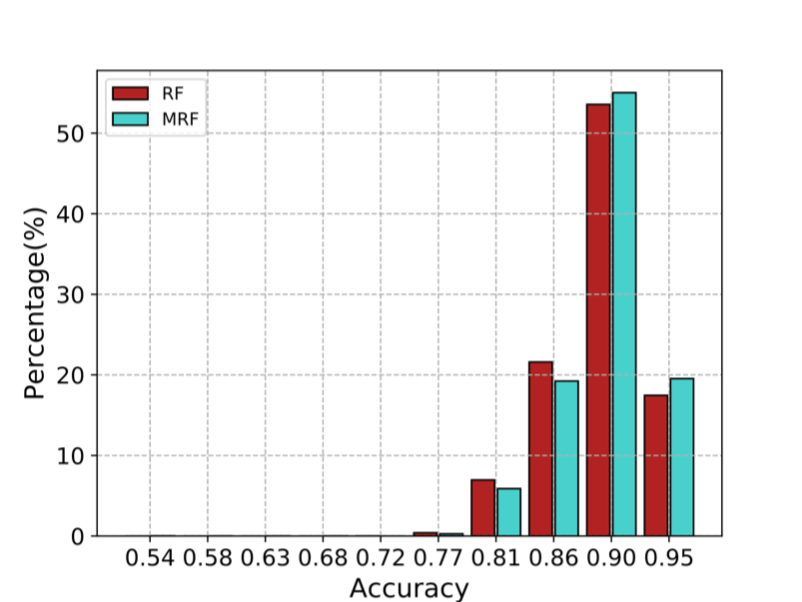}}
\subfigure[CAR]{
\label{fig3:subfig:e}
\includegraphics[width=0.28\textwidth]{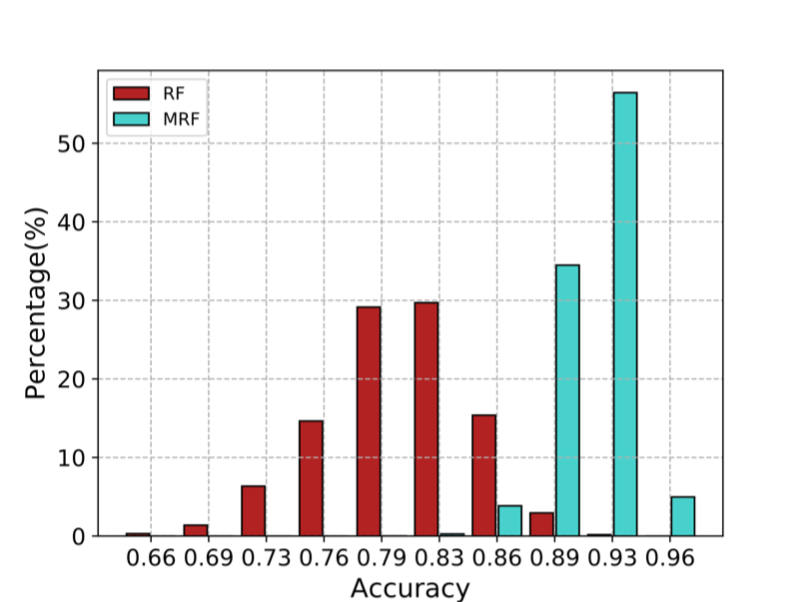}}
\subfigure[CONNECT-4]{
\label{fig3:subfig:f}
\includegraphics[width=0.28\textwidth]{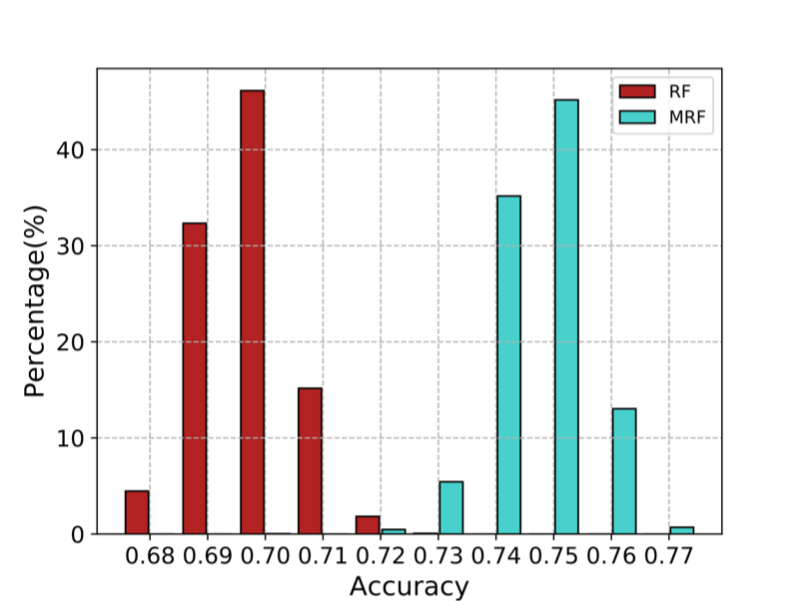}}
\caption{Distribution of performance on individual trees in MRF and RF.}
\label{fig3} 
\vskip -0.1in
\end{figure*}

\subsection{Performance on Individual Trees}
We further investigated why the MRF achieves such good performance by studying the performance of MRF on individual trees. In our 10 times 10-fold cross-validation, 10,000 trees were generated for each method. We compared the distribution of prediction accuracy over those trees between the MRF and the standard RF. Figure \ref{fig3} displays the distributions for six datasets.

The tree-level performance of the MRF is generally better than that of the standard RF, which verifies the superiority of the multinomial-based random split process. However, we also have to notice that a good performance over individual trees does not necessarily lead to a good performance of a forest, since the performance of a forest may also be affected by the diversity of the trees. For example, the MRF has a significantly better tree-level performance on the CAR dataset, whereas its forest-level performance is not significantly different from the standard RF. 

Although we have not been able to make a direct connection between the overall performance and the performance on individual trees, understanding the complexity of the relationship is still meaningful. The specific connection will be explored in the future work.

\end{appendices}

\end{document}